\documentclass[oneside,11pt]{article} 
\renewenvironment{abstract}
  {{\centering\large\bfseries Abstract\par}\vspace{0.7ex}%
    \bgroup
       \leftskip 20pt\rightskip 20pt\small\noindent\ignorespaces}%
  {\par\egroup\vskip 0.25ex}

\newenvironment{keywords}
{\vspace{0.05in}\bgroup\leftskip 20pt\rightskip 20pt \small\noindent{\bfseries
Keywords:} \ignorespaces}
{\par\egroup\vskip 0.25ex}

\usepackage{times}
\usepackage{graphicx}
\usepackage{algorithm,algorithmic}
\usepackage{amsfonts,amssymb,amsmath,amsthm} 
\usepackage{xcolor}
\usepackage{url}
\usepackage{booktabs}
\usepackage{enumitem}
\usepackage[colorlinks,
            linkcolor=blue,
            citecolor=blue,
            urlcolor=magenta,
            linktocpage,
            plainpages=false]{hyperref}

\usepackage{natbib}
\setcitestyle{authoryear,open={(},close={)}}

\usepackage{multirow}

\theoremstyle{plain}
\newtheorem{theorem}{Theorem}
\newtheorem{lemma}{Lemma}
\newtheorem{proposition}{Proposition}
\newtheorem{corollary}{Corollary}

\newcommand{\defeq}{\stackrel{\text{def}}{=}}

\renewcommand{\(}{\left(}
\renewcommand{\)}{\right)}

\DeclareMathOperator*{\argmin}{arg\,min}

\DeclareMathOperator*{\tr}{Tr}
\DeclareMathOperator*{\rank}{rank}

\newcommand{\st}{\mathrm{s.t.}}

\newcommand{\bX}{{X}}
\newcommand{\bY}{{Y}}
\newcommand{\bZ}{{Z}}
\newcommand{\bE}{{E}}
\newcommand{\bL}{{L}}

\newcommand{\bU}{{U}}
\newcommand{\bV}{{V}}
\newcommand{\bA}{{A}}
\newcommand{\bB}{{B}}

\newcommand{\bD}{{D}}
\newcommand{\bM}{{M}}

\newcommand{\bG}{{G}}
\newcommand{\bI}{{I}}
\newcommand{\bW}{{W}}
\newcommand{\bH}{{H}}

\newcommand{\bx}{\boldsymbol{x}}
\newcommand{\by}{\boldsymbol{y}}
\newcommand{\bz}{\boldsymbol{z}}
\newcommand{\be}{\boldsymbol{e}}
\newcommand{\br}{\boldsymbol{r}}
\newcommand{\bu}{\boldsymbol{u}}
\newcommand{\bv}{\boldsymbol{v}}
\newcommand{\ba}{\boldsymbol{a}}
\newcommand{\bb}{\boldsymbol{b}}

\newcommand{\bd}{\boldsymbol{d}}
\newcommand{\bm}{\boldsymbol{m}}
\newcommand{\Rq}{\mathbb{R}^q}
\newcommand{\Rp}{\mathbb{R}^p}

\newcommand{\Rpd}{\mathbb{R}^{p\times d}}

\newcommand{\Rdd}{\mathbb{R}^{d\times d}}
\newcommand{\Rnd}{\mathbb{R}^{n\times d}}
\newcommand{\Rdn}{\mathbb{R}^{d\times n}}
\newcommand{\Rpn}{\mathbb{R}^{p\times n}}
\newcommand{\Rnn}{\mathbb{R}^{n\times n}}

\newcommand{\twonorm}[1]{\left\lVert #1 \right\rVert_{2}}
\newcommand{\onenorm}[1]{\left\lVert #1 \right\rVert_{1}}
\newcommand{\abs}[1]{\left\lvert #1 \right\rvert}

\newcommand{\nuclearnorm}[1]{\left\lVert #1 \right\rVert_{*}}
\newcommand{\fronorm}[1]{\left\lVert #1 \right\rVert_{F}}
\newcommand{\spenorm}[1]{\left\lVert #1 \right\rVert}
\newcommand{\infnorm}[1]{\left\lVert #1 \right\rVert_{\infty}}

\newcommand{\fractwo}[1]{\frac{#1}{2}}

\newcommand{\tl}{\tilde{\ell}}
\renewcommand{\th}{\tilde{h}}

\newcommand{\EXP}{\mathbb{E}}
\newcommand{\trans}{^\top}

\newcommand{\const}{\mathrm{C}}

\newcommand{\bzero}{\boldsymbol{0}}

\graphicspath{{fig/}}

\title{Efficient Online Minimization for Low-Rank Subspace Clustering}

\author{
{\bf Jie Shen}\\
Dept. of Computer Science\\
School of Arts and Sciences\\
Rutgers University\\
Piscataway, NJ 08854, USA\\
\texttt{js2007@rutgers.edu}
\and
{\bf Ping Li}\\
Dept. of Statistics \& Biostatistics\\
Department of Computer Science\\
Rutgers University\\
Piscataway, NJ 08854, USA\\
\texttt{pingli@stat.rutgers.edu}
\and
{\bf Huan Xu}\\
Industrial \& Systems Engineering\\
Georgia Institute of Technology\\
Atlanta, GA 30332, USA\\
\texttt{huan.xu@isye.gatech.edu}
}
\date{}

\begin{document}
\maketitle

\begin{abstract}
Low-rank representation~(LRR) has been a significant method for segmenting data that are generated from a union of subspaces. It is, however, known that solving the LRR program is challenging in terms of time complexity and memory footprint, in that the size of the nuclear norm regularized matrix is $n$-by-$n$ (where $n$ is the number of samples). In this paper, we thereby develop a fast online implementation of LRR that reduces the memory cost from $O(n^2)$ to $O(pd)$, with $p$ being the ambient dimension and $d$ being some estimated rank~($d < p \ll n$). The crux for this end is a non-convex reformulation of the LRR program, which pursues the basis dictionary that generates the (uncorrupted) observations. We build the theoretical guarantee that the sequence of the solutions produced by our algorithm  converges to a stationary point of the empirical and the expected loss function asymptotically. Extensive experiments on synthetic and realistic datasets further substantiate that our algorithm is fast, robust and memory efficient.
\end{abstract}



\section{Introduction}\label{sec:intro}
In modern scientific computing, data are routinely generated from a union of small subspaces for which, traditional tools such as principal component analysis are not able to identify the group structure. As an alternative, in the past a few years, subspace clustering~\citep{vidal2010tutorial,soltanolkotabi2012geometric} has been extensively studied and has established solid applications in, for example, computer vision~\citep{elhamifar2009sparse} and network topology inference~\citep{eriksson2011high}.  Among many subspace clustering methods which seek a structured representation  to fit the underlying data, two prominent examples are sparse subspace clustering~(SSC)~\citep{elhamifar2009sparse,soltanolkotabi2014robust} and low-rank representation~(LRR)~\citep{liu2013robust}. Both of them utilize the idea of self-expressiveness, i.e., expressing each sample as a linear combination of the remaining. Hence, the coefficients quickly suggest which data points belong to the same group (i.e., subspace). What is of difference is that SSC pursues a sparse solution, i.e., using a small fraction of samples to represent each data point. In this light, not only the clustering structure can be found by SSC, but also the most correlated samples are detected. For LRR, it prefers a (global) low-rank structure. This is motivated by many applications where the union of the subspaces is still of low rank. An appealing property of LRR is that it comes up with the theoretical guarantee of recovering the true low-rank model in addition to clustering the samples.

In this paper, we are interested in the LRR method, which is shown to achieve the state-of-the-art performance on a broad range of real-world problems~\citep{liu2013robust}. Recently, \citet{liu2014recovery} demonstrated that, when equipped with a proper dictionary, LRR can even handle the coherent data~--~a challenging issue in the literature~\citep{candes2009exact,candes2011robust} but is ubiquitous in realistic datasets such as the Netflix movie rating problem.

Formally, the LRR program we investigate here is formulated as follows~\citep{liu2013robust}:
\begin{align}
	\label{eq:lrr}
	\min_{\bX, \bE}\ \fractwo{\lambda_1} \fronorm{\bZ-\bY\bX-\bE}^2 + \nuclearnorm{\bX} + \lambda_2 \onenorm{\bE}.
\end{align}
Here, $\bZ = (\bz_1, \bz_2, \cdots, \bz_n) \in \Rpn$ is the observation matrix with $n$ samples, each of which is a $p$-dimensional column vector. The matrix $\bY \in \Rpn$ is a given dictionary, $\bE$ is some possible sparse corruption and $\lambda_1 > 0$ and $\lambda_2 > 0$ are two tunable parameters. Typically, $\bY$ is chosen as the dataset $\bZ$ itself, hence the idea of self-expressiveness. As $\bX$ is penalized by the nuclear norm which is a convex surrogate of the rank function, the program seeks a low-rank representation among all samples, each of which can be approximated by a linear combination of the atoms in the dictionary $\bY$. Note that a column of $\bX$ is the coefficients associated with an individual sample. When the optimal solution is obtained, one may perform spectral clustering~\citep{ng2002spectral} on $\bX$ since the entries reflect the correlation between the (uncorrupted) data.

To aid intuition and to introduce more background for the LRR program~\eqref{eq:lrr}, we discuss the ability of subspace clustering of LRR in the noiseless case. That is, the data matrix $\bZ$ is generated from a union of (low-rank) subspaces. Since each sample can be represented only by those belonging to the same subspace, one may simply solve the linear system
\begin{align*}
\bZ = \bZ \bX,
\end{align*}
which certainly produces a block diagonal solution of $\bX$ as long as the small subspaces are disjoint. That being said, a careful reader may observe that the above linear system does not exclude the trivial identity matrix $\bX = \bI_n$. In order to alleviate it and to conform the low-rank structure of the data, LRR looks for the one with the lowest rank, i.e.,
\begin{align}\label{eq:lrr noiseless}
\min_{\bX}\ \nuclearnorm{\bX},\ \st\ \bZ = \bZ\bX.
\end{align}
The following lemma, which is due to~\citet{liu2010lrr}, justifies that the above program correctly segments the data.
\begin{lemma}
Suppose that there are $k$ number of small subspaces and without loss of generality, the data matrix is organized as $\bZ = (\bZ_1, \bZ_2, \cdots, \bZ_k)$, where $\bZ_i \in \mathbb{R}^{p \times n_i}$ is the collection of samples coming from the $i$th subspace. Further assume that there are sufficient samples for each subspace such that $n_i > \rank(\bZ_i)$. If the subspaces are mutually disjoint, then there exists an optimal solution to~\eqref{eq:lrr noiseless},
\begin{align*}
\bX^* =
\begin{pmatrix}
\bX^*_1 & 0 & 0 & 0\\
0 & \bX_2^* & 0 & 0\\
0 & 0 & \ddots & 0\\
0 & 0 & 0 & \bX_k^*
\end{pmatrix}
\end{align*}
such that $\rank(\bX_i^*) = \rank(\bZ_i)$ for all $1 \leq i \leq k$.
\end{lemma}

More generally, one may think of the noisy model $\bZ = \bY \bX + \bE + \bG$ for some given dictionary $\bY$, a sparse corruption $\bE$ and a Gaussian noise $\bG$. In this case, program~\eqref{eq:lrr} is a straightforward extension to achieve the robustness. While a large body of work has shown that LRR is able to segment the data (see, e.g.,~\citet{liu2014recovery}), three issues are immediately incurred for LRR in the face of big data:
\begin{enumerate}[label=$(I\arabic*)$]
\item Memory cost of $\bX$. In the LRR formulation~\eqref{eq:lrr}, there is typically no sparsity assumption on $\bX$. Hence, the memory footprint of $\bX$ is proportional to $n^2$ which precludes most of the recently developed nuclear norm solvers~\citep{lin2010augmented,jaggi2010simple,avron2012efficient,hsieh2014nuclear}.\label{is:memX}

\item Computational cost of $\nuclearnorm{\bX}$. Since the size of the nuclear norm regularized matrix $\bX$ is $n \times n$, optimizing such problems can be computationally expensive even when $n$ is moderate, say $n = 1000$~\citep{recht2010guaranteed}.\label{is:compX}

\item Memory cost of $\bY$. As the size of  the dictionary $\bY$ is proportional to sample size $n$, it is prohibitive to store the entire dictionary $\bY$ during optimization when $n$ is large.\label{is:memY}
\end{enumerate}
To remedy these problems, especially the memory bottleneck, one potential way is solving the LRR program in an online manner. That is, we sequentially reveal the samples $\bz_1, \bz_2, \cdots, \bz_n$ and update the components in $\bX$ and $\bE$. Nevertheless, it turns out that the LRR program~\eqref{eq:lrr} is not suitable for online minimization. To see this, we note that each column of $\bX$ is the coefficients of a sample with respect to the {\em entire} dictionary $\bY$, for example, $\bz_1 \approx \bY \bx_1 + \be_1$. This indicates that without further technique, we have to load the entire dictionary $Y$ so as to update $\bx_1$ and $\be_1$. Again, this yields an $O(n)$ memory cost. Hence, for our purpose, we need to tackle a {more serious} challenge:
\begin{enumerate}[label=$(I4)$]
\item Partial realization of $\bY$. We are required to guarantee the optimality of the solution but can only access part of the atoms of $\bY$ in each iteration.\label{is:partialY}
\end{enumerate}

\subsection{Contribution} 
In this paper, henceforth, we propose a new algorithm termed online low-rank subspace clustering (OLRSC), which admits a low computational complexity. Compared to existing solvers, OLRSC reduces the memory cost of LRR from $O(n^2)$ to $O(pd)$ where $d$ is an estimated rank of the uncorrupted data~($d < p \ll n$). This nice property makes OLRSC an appealing solution for large-scale subspace clustering problems. Furthermore, we prove that the sequence of the solutions produced by OLRSC converges to a stationary point of the expected loss function asymptotically even though one atom of $\bY$ is available at each iteration. In a nutshell, OLRSC resolves {\em all} practical issues of LRR and still promotes global low-rank structure~--~the merit of LRR. Finally, concerning the robustness of the algorithm, our empirical study suggests accurate recovery of the true subspace even when a large fraction of the entries are corrupted.


The present paper considerably extends our earlier conference version~\citep{shen2016online}. In particular, we make extra efforts to remove the assumption that the given dictionary $\bY$ has to be full-rank, which indicates that our algorithm can be applied to a broader range of problems. We also propose refined conditions on parameter scaling under which the algorithm converges. Lastly, we draw additional connections to the literature.

\subsection{Related Work}
Low-rankness has been studied for more than two decades. As one of the earliest work, \citet{fazel2001rank} appealed to a rank minimization program for system identification and signal processing. It was further suggested that the nuclear norm (also known as trace norm) is a good convex surrogate to the rank of a matrix. Such a key observation  was utilized in machine learning for recommender systems~\citep{srebro2004mmmf,srebro2005fast}. Though empirically effective, it was not clear under which conditions such a nuclear norm based program exactly returns the true low-rank solution. In the seminal work of~\citet{candes2011robust}, it was shown that if the singular vectors of the true low-rank matrix does not concentrate around the canonical basis, formally referred to as the incoherence property, then a convex program with proper parameter setting recovers the underlying matrix even a constant fraction of the entries are corrupted. Following~\citet{candes2011robust}, a considerable number of work extends the low-rank model to accommodate outliers~\citep{xu2013outlier}, high-dimensional case~\citep{xu2010principal}, graph clustering~\citep{chen2014clustering}, to name a few. This work considers the variant of multiple subspaces, which is due to~\citet{liu2010lrr}. In particular, in place of drawing insight on the statistical performance, we mainly focus on an efficient and provable algorithm for subspace clustering. There is a plethora of work attempting to mitigate the memory and computational bottleneck of the nuclear norm regularizer. However, to the best of our knowledge, none of them can handle Issue~\ref{is:memY} and Issue~\ref{is:partialY}.

One of the most popular ways to alleviate the huge memory cost is online implementation. \citet{feng2013online} devised an online algorithm for the robust principal component analysis~(RPCA) problem, which makes the memory cost independent of the sample size. Yet, compared to RPCA where the size of the nuclear norm regularized matrix is $p \times n$, that of LRR is $n\times n$~--~a worse and more challenging case. Moreover, their algorithm cannot address the partial dictionary issue that emerges in our case.

To tackle the computational overhead, \citet{jaggi2010simple} utilized a sparse semi-definite programming solver to derive a simple yet efficient algorithm. Unfortunately, the memory requirement of their algorithm is proportional to the number of observed entries, making it impractical when the regularized matrix is large and dense (which is the case of LRR). \citet{avron2012efficient} combined stochastic subgradient and incremental SVD to boost efficiency. But for the LRR problem, the type of the loss function does not meet the requirements and thus, it is still not practical to use that algorithm in our case.

Another line in the literature explores a structured formulation of LRR beyond the low-rankness. For example, \citet{wang2013provable} provably showed that combining LRR and SSC can take advantages of both methods. Whereas \citet{wang2013provable} promotes the sparsity on the representation matrix $\bX$, \citet{shen2016learning} demonstrated how to pursue a sparsity structure on the factors of $\bX$, which is known to be more flexible~\citep{haeffele2014structured}. Due to the non-convexity of matrix factorization, a large body of work is dedicated to characterizing the conditions under which gradient descent ensures global optimum~\citep{jain2013low,lee2016gradient,ge2016matrix}.

Subsequent to the conference version, the authors made use of similar idea to present a structured low-rank representation framework~\citep{shen2016learning}. Inspired by the unified framework, \citet{zhao2017unified} presents a non-asymptotic algorithm to solve the subspace clustering problem.

\subsection{Overview}
In Section~\ref{sec:setup}, we reformulate the LRR program~\eqref{eq:lrr} so that it is amenable to online optimization. The detailed algorithm is presented in Section~\ref{sec:alg}, along with a theoretical justification on its convergence behavior.  We illustrate the efficiency and efficacy of the proposed algorithm by performing extensive experiments in Section~\ref{sec:exp}. Finally, we conclude the paper in Section~\ref{sec:conclusion}. All the proofs can be found in the appendix.

\subsection{Notation}
We use bold lowercase letters, {e.g.,} $\bv$, to denote column vectors. The $\ell_2$ norm and $\ell_1$ norm of a vector $\bv$ are denoted by $\twonorm{\bv}$ and $\onenorm{\bv}$ respectively. Capital letters such as $\bM$ are used to denote a matrix, and its transpose is denoted by $\bM\trans$. For an invertible matrix $\bM$, we write its inverse as $\bM^{-1}$. The capital letter $\bI$ is reserved for the identity matrix. When necessary, we write $\bI_r$ to emphasize its size. The $j$th column of a matrix $\bM$ is denoted by $\bm_j$, and we write $m_{ij}$ for the $(i, j)$-th component of $\bM$. Three matrix norms  will be used: $\nuclearnorm{\bM}$ for the nuclear norm which is the sum of the singular values of $\bM$, $\fronorm{\bM}$ for the Frobenius norm and {$\onenorm{\bM}$ for the $\ell_1$ norm of a matrix seen as a long vector.} The trace of a square matrix $\bM$ is denoted by $\tr(\bM)$. 

For an integer $n$, we use $[n]$ to denote the integer set $\{1, 2, \cdots, n\}$. We follow the standard asymptotic notation. That is, for two functions $f(n)$ and $g(n)$, we use $f(n) = O(g(n))$ to denote that there exists an absolute constant $\const$ such that $f(n) \leq \const \cdot g(n)$ as $n$ goes to infinity. We also write $f(n) = \Omega(g(n))$ if $f(n) \geq \const \cdot g(n)$. When $\const_1 \cdot g(n) \leq f(n) \leq \const_2 \cdot g(n)$, we write $f(n) = \Theta(g(n))$.

\section{Problem Formulation and Algorithm}\label{sec:setup}

Recall the LRR program given in~\eqref{eq:lrr}. Our goal is to efficiently learn the representation matrix $\bX$ and the corruption matrix $\bE$ in an online manner so as to mitigate the issues mentioned in Section~\ref{sec:intro}. To this end, we reformulate it as an empirical risk minimization problem which is amenable for online optimization. 

The first technique for our purpose is a {\em non-convex reformulation} of the nuclear norm. Assume that the rank of $\bX$ is at most $d$. Then~\citet{fazel2001rank} showed that,
\begin{equation}
\label{eq:nuclear reform}
\nuclearnorm{\bX} = \min_{\bU, \bV, \bX=\bU\trans \bV} \fractwo{1} \( \fronorm{\bU}^2 + \fronorm{\bV}^2 \),
\end{equation}
where $\bU \in \Rdn$ and $\bV \in \Rdn$. The minimum can be attained at, for example, $\bU\trans = \bU_0 S_0^{1/2}$ and $\bV\trans = \bV_0 S_0^{1/2}$ where $\bX = \bU_0 S_0 \bV_0\trans$ is the singular value decomposition. In this way, \eqref{eq:lrr} can be written as follows:
\begin{align}\label{eq:nuclear imme}
\min_{\bU, \bV, \bE} \fractwo{\lambda_1}\fronorm{\bZ-\bY\bU\trans\bV -\bE}^2 + \fractwo{1} \fronorm{\bU}^2 + \fractwo{1}\fronorm{\bV}^2  + \lambda_2 \onenorm{\bE}.
\end{align}
Note that by this reformulation, updating the entries in $\bX$ amounts to sequentially updating the columns of $\bU$ and $\bV$, as shown below:
\begin{align*}
\underbrace{ \begin{pmatrix}
x_{11} & x_{12} & \cdots & x_{1n} \\
x_{21} & x_{22} & \cdots & x_{2n}\\
\vdots & \vdots & \ddots & \vdots\\
x_{n1} & x_{n2} & \cdots & x_{nn}
\end{pmatrix}}_{\bX \in \Rnn}
= \underbrace{\begin{pmatrix}
\bu_1\trans\\
\bu_2\trans\\
\vdots\\
\bu_n\trans
\end{pmatrix}}_{\bU\trans \in \Rnd}
\times
\underbrace{\begin{pmatrix}
\bv_1 & \bv_2 & \cdots & \bv_n
\end{pmatrix}}_{\bV \in \Rdn}.
\end{align*}
For instance, when we have $\{ \bu_i \}_{i=1}^t$ and $\{ \bv_j \}_{j=1}^t$ on hand, we are able to recover the components $\{ x_{ij} \}_{1\leq i, j \leq t}$ since each $x_{ij}$ equals $\bu_i\trans\bv_j$. It is worth mentioning that this technique is utilized in~\citet{feng2013online} for online RPCA.  Unfortunately, the size of $\bU$ and $\bV$ in our problem are both proportional to the sample size $n$ and the dictionary $\bY$ is partially observed in each iteration, making the algorithm in~\citet{feng2013online} not applicable to LRR. Related to the online implementation, another challenge is that, all the columns of $\bU$ are coupled together at this moment as $\bU\trans$ is left multiplied by $\bY$ in the first term of~\eqref{eq:nuclear imme}. Since we do not want to load the entire dictionary $\bY$, this makes it difficult to sequentially compute the columns of $\bU$.

For the sake of decoupling the columns of $\bU$, as part of the crux of our techniques, we introduce an {auxiliary variable} $\bD = \bY\bU\trans$, whose size is $p \times d$ ({i.e.}, independent of the sample size $n$). Interestingly, in this way, we are approximating  the term $\bZ - \bE$  with $\bD\bV$, which provides an intuition on the role of $\bD$: namely, $\bD$ can be seen as a {\em basis dictionary} of the clean data, with $\bV$ being the coefficients.\footnote{The matrix $\bY$ is referred to as atom dictionary, and $\bD$ is referred to as basis dictionary.}

These key observations allow us to derive an equivalent reformulation to LRR~\eqref{eq:lrr}:
\begin{align*}
\min_{\bD, \bU, \bV, \bE}\ \fractwo{\lambda_1}\fronorm{\bZ-\bD\bV -\bE}^2 + \fractwo{1} \( \fronorm{\bU}^2 + \fronorm{\bV}^2 \)  + \lambda_2 \onenorm{\bE},\quad \st\ \bD = \bY\bU\trans.
\end{align*}
By penalizing the constraint in the objective, we obtain a {\em regularized} version of LRR on which our new algorithm is based:
\begin{align}
\label{eq:reg batch}
\min_{\bD, \bU, \bV, \bE}&\ \fractwo{\lambda_1} \fronorm{ \bZ - \bD\bV -\bE }^2 + \fractwo{1}\( \fronorm{\bU}^2 + \fronorm{\bV}^2 \)  + \lambda_2 \onenorm{\bE} + \fractwo{\lambda_3} \fronorm{\bD - \bY\bU\trans}^2.
\end{align}
We note that there are two advantages of \eqref{eq:reg batch} compared to~\eqref{eq:lrr}. First, it is amenable for online optimization. Second, it is more informative since it explicitly models the basis of the union of subspaces, which yields a better subspace recovery and clustering to be shown in Section~\ref{sec:exp}. 

We also point out that due to our explicit modeling of the basis, we unify  LRR and RPCA as follows: for LRR, $\bD \approx \bY\bU\trans$ (or $\bD = \bY\bU\trans$ if $\lambda_3$ tends to infinity) while for RPCA, $\bD = \bU\trans$. That is, ORPCA~\citep{feng2013online} considers a problem of $\bY = \bI_p$ whose size is independent of $n$, hence can be kept in memory which naturally resolves Issue~\ref{is:memY} and~\ref{is:partialY}. This is why RPCA can be easily implemented in an online fashion while LRR cannot.


Now we return to the online implementation of LRR. The main idea is optimizing a surrogate of the empirical risk function (to be defined) in each iteration, which is fast and memory efficient. Let $\bz_i$, $\by_i$, $\be_i$, $\bu_i$, and $\bv_i$ be the $i$th column of matrices $\bZ$, $\bY$, $\bE$, $\bU$ and $\bV$ respectively and let
\begin{align}
\tl(\bD, \bv, \be; \bz) &\defeq \fractwo{\lambda_1} \twonorm{ \bz - \bD \bv - \be }^2 + \fractwo{1}\twonorm{\bv}^2 + \lambda_2 \onenorm{\be}, \notag\\
\label{eq:ell}
\ell(\bD; \bz) &\defeq \min_{\bv, \be} \tl(\bD, \bv, \be; \bz).
\end{align}
Further, we define
\begin{align}
\th(\bD, \bU; \bY) &\defeq \sum_{i=1}^n \fractwo{1} \twonorm{\bu_i}^2 + \fractwo{\lambda_3} \fronorm{\bD - \sum_{i=1}^{n} \by_i \bu_i\trans}^2,\notag\\
\label{eq:h}
h(\bD; \bY) &\defeq \min_{\bU} \th(\bD, \bU; \bY).
\end{align}
Then \eqref{eq:reg batch} can be rewritten as follows:
\begin{align}\label{eq:tmp}
\min_{\bD} \min_{\bU, \bV, \bE}\ \sum_{i=1}^{n} \tl(\bD, \bv_i, \be_i; \bz_i) + \th(\bD, \bU; \bY),
\end{align}
which amounts to minimizing the {\em empirical loss} function:
\begin{align}
\label{eq:f_n(D)}
\min_{\bD} f_n(\bD) \defeq \frac{1}{n}\sum_{i=1}^{n} \ell(\bD; \bz_i) + \frac{1}{n} h(\bD; \bY).
\end{align}

\subsection{Expected Loss}
{In stochastic approximation, we are also interested in analyzing the optimality of the obtained solution with respect to the expected loss function~\citep{bottou2007tradeoffs}. To this end, we first derive the optimal solutions $\tilde{\bU}$, $\tilde{\bV}$ and $\tilde{\bE}$ that minimize~\eqref{eq:tmp} which renders a concrete form of the empirical loss function $f_n(\bD)$.

Given $\bD$, we need to compute the optimal solutions $\tilde{\bU}$, $\tilde{\bV}$ and $\tilde{\bE}$ to evaluate the objective value of $f_n(\bD)$. What is of interest here is that, the optimization procedure of $\bU$ is quite different from that of $\bV$ and $\bE$. According to \eqref{eq:ell}, when $\bD$ is given, each $\tilde{\bv}_i$ and $\tilde{\be}_i$ can be solved by only accessing the $i$th sample $\bz_i$. However, the optimal $\tilde{\bu}_i$ depends on the whole dictionary $\bY$ as the second term in $\th(\bY, \bD, \bU)$ couples all the $\bu_i$'s. Fortunately, we can obtain a closed form solution for each $\tilde{\bu}_i$, as stated below.
\begin{proposition}\label{prop:opt U}
Suppose that $\bY$ and $\bD$ are fixed. Then the optimal solution $\tilde{\bU} = (\tilde{\bu}_1, \tilde{\bu}_2, \dots, \tilde{\bu}_n)$ that minimizes $\th(\bD, \bU; \bY)$ and Eq.~\eqref{eq:reg batch} is given by
\begin{align}\label{eq:actual u_t}
\tilde{\bu}_i = \bD\trans \(\frac{1}{\lambda_3 }\bI_p + \bY\bY\trans \)^{-1} \by_i,\ \forall\ i \in [n].
\end{align}
Hence,
\begin{align}
h(\bD; \bY) = \frac{1}{2} \fronorm{\bD\trans \(\frac{1}{\lambda_3}\bI_p + \bY\bY\trans\)^{-1} \bY}^2 + \frac{1}{2\lambda_3} \fronorm{  \( \frac{1}{\lambda_3} \bI_p + \bY\bY\trans \)^{-1}\bD }^2.
\end{align}
\end{proposition}
The result follows by noting the first order optimality condition and some basic algebra. The proof can be found in  Appendix~\ref{supp:sec:opt u}.

Let $\sigma_i(\bY)$ be the singular values of the matrix $\bY \in \Rpn$ where $1 \leq i \leq p$. By calculation, it is not hard to see that
\begin{align*}
&\ \sigma_{\max}\( \(\frac{1}{\lambda_3}\bI_p + \bY\bY\trans\)^{-1} \) = \frac{1}{(\sigma_p(\bY))^2 + \lambda_3^{-1}} \leq \lambda_3,\\
&\ \sigma_{\max} \(\(\frac{1}{\lambda_3}\bI_p + \bY\bY\trans\)^{-1} \bY\) = \frac{\sigma_i(\bY)}{( \sigma_i(\bY) )^2 + \lambda_3^{-1}} \leq \frac{\sqrt{\lambda_3}}{2}\ \text{for\ some}\ i.
\end{align*}
Thereby, we obtain the upper bound
\begin{align*}
h(\bD; \bY) \leq   \frac{\lambda_3}{8} \fronorm{\bD}^2 + \frac{\lambda_3}{2} \fronorm{\bD}^2 \leq \frac{5\lambda_3}{8} \fronorm{\bD}^2.
\end{align*}
Suppose that $\bD$ is fixed. Also notice that the size of $\bD$ does not grow with $n$. It follows that
\begin{align*}
\lim_{n \rightarrow \infty}\ \frac{1}{n} h(\bD; \bY) = 0,
\end{align*}
if $\lambda_3$ is such that $\lim_{n\rightarrow \infty} \lambda_3/n = 0$. Hence, asymptotically, minimizing the empirical loss $f_n(\bD)$ amounts to optimizing $n^{-1} \sum_{i=1}^{n} \ell(\bD; \bz)$. Note that when $\lambda_1$ and $\lambda_2$ are independent of $n$, we quickly get
\begin{align}\label{eq:E}
\lim_{n \rightarrow \infty} \frac{1}{n} \sum_{i=1}^{n} \ell(\bD; \bz_i) = \EXP_{\bz} [\ell(\bD; \bz)],
\end{align}
assuming that all the samples are drawn i.i.d. from some unknown distribution. This gives the {\em expected loss} function
\begin{align}
\label{eq:f(D)}
f(\bD) \defeq  \EXP_{\bz}[\ell(\bz, \bD) ]= \lim_{n \rightarrow \infty} f_n(\bD) .
\end{align}


\subsection{Algorithm}\label{sec:alg}
We are now in the position to elaborate an online optimization algorithm for low-rank subspace clustering. Namely, we show how to solve~\eqref{eq:f_n(D)} in an efficient manner. From a high level, we alternatively minimize over the variables $\bD$ and $\{\bu_i, \bv_i, \be_i\}_{i=1}^n$. That is, we begin with an initial (and inaccurate) guess of the variable $\bD$, and solve the subproblems~\eqref{eq:ell} and~\eqref{eq:h}. Using the obtained solution, we construct a surrogate function of the empirical loss~\eqref{eq:f_n(D)}, which is further optimized to refine our initial guess of $\bD$. Such a paradigm was utilized in~\citet{mairal2010online,mairal2013stochastic,shen2014online} for different problems. It turns out that the case studied here is more challenging due to the component $h(\bD; \bY)$ of the empirical loss function. To optimize it in an online manner and to analyze the performance, we develop novel techniques that will be described in the following.

\subsubsection{Optimize $\bv$ and $\be$}
For exposition, we first investigate the sum of $\ell(\bD; \bz_i)$ where $1 \leq i \leq n$ in~\eqref{eq:f_n(D)}. This is where the variables $\bv_i$ and $\be_i$ are involved. Suppose that at the $t$-th iteration, a fresh sample $\bz_t$ is drawn and we have an initial guess of $\bD$, say $\bD_{t-1}$. Then it is easy to see that the optimal solution $\{\bv_t^*, \be_t^*\}$ that minimizes $\tl(\bD_{t-1}, \bv, \be; \bz_t)$ is given by
\begin{align*}
\{\bv_t^*, \be_t^*\} = \argmin_{\bv, \be}\ \fractwo{\lambda_1} \twonorm{ \bz_t - \bD_{t-1} \bv - \be }^2 + \fractwo{1}\twonorm{\bv}^2 + \lambda_2 \onenorm{\be}.
\end{align*}
Since the objective function is jointly strongly convex over the variables $\{\bv, \be\}$, one may apply coordinate descent to obtain the optimum~\citep{bertsekas1999nonlinear}. In particular, we observe that if $\be$ is fixed, we can optimize $\bv$ in closed form:
\begin{align}\label{eq:v}
\bv = (\bD_{t-1}\trans \bD_{t-1} + \bI_d / \lambda_1)^{-1}\bD_{t-1}\trans (\bz_t - \be).
\end{align}
Conversely, given $\bv$, the variable $\be$ is obtained via soft-thresholding~\citep{donoho1995denoising}:
\begin{align}\label{eq:e}
\be = \mathcal{S}_{\lambda_2 / \lambda_1}[\bz_t - \bD_{t-1}\bv].
\end{align}
See Algorithm~\ref{alg:ve} in Appendix~\ref{supp:sec:alg} for details. It follows that $\tl(\bD_{t-1}, \bv_t^*, \be_t^*; \bz_t)$ is a surrogate to $\ell(\bD; \bz_t)$, since $\bD_{t-1}$ is our guess of the optimal $\bD$. 

\subsubsection{Optimize $\bu$}
We move on to describe how to optimize the second component $h(\bD; \bY)$, which contains the variables $\{\bu_i\}_{i \geq 1}$. Again, we will assume that an initial guess of $\bD$ is available, so we only need to minimize a surrogate. However, looking at the solution~\eqref{eq:actual u_t}, we notice that even $\bD$ is given (e.g., some initial guess $\bD_{t-1}$), we have to load the {\em entire} dictionary $\bY \in \Rpn$ so as to obtain the local optimum $\tilde{\bu}_t$, and hence a surrogate of $h(\bD; \bY)$. This is quite different from the scheme of optimizing $\bv$ and $\be$, where optimum only depends on the new sample. Indeed, though the framework of our algorithm follows from~\citet{mairal2010online,feng2013online,shen2014online}, none of them considers the situation as we stated here (the optimum in their work can be computed directly when a new sample arrives).

In order to remedy the issue, our novelty here is constructing a proxy function by minimizing which, we are generating a solution that gradually approximates to the one of $\th(\bD, \bU; \bY)$. The proxy function is given as follows
\begin{align}
\label{eq:tl_2}
\tl_2(\bD, \bu; \bM, \by)  \defeq \fractwo{1} \twonorm{\bu}^2 + \fractwo{\lambda_3} \fronorm{\bD - \bM - \by \bu\trans}^2.
\end{align}
Suppose that we store the accumulation matrix
\begin{align}\label{eq:M_t}
\bM_{t-1} \defeq \sum_{i=1}^{t-1} \by_i (\bu_i^*)\trans \in \Rpd,
\end{align}
where we define $M_0$ as a zero matrix. At the $t$-th iteration, when a new atom $\by_t$ is revealed, we calculate $\bu_t^*$ as the minimizer of $\tl_2(\bD_{t-1}, \bu; \bM_{t-1}, \by_t)$, for which we have the closed form solution
\begin{align}\label{eq:u_t}
\bu_t^* = (\twonorm{\by_t}^2 + {1}/{\lambda_3})^{-1} (\bD_{t-1} - \bM_{t-1})\trans \by_t.
\end{align}
The solution in~\eqref{eq:u_t} differs from that of~\eqref{eq:actual u_t} a lot. Though \eqref{eq:u_t} is not an accurate solution that minimizes $\th(\bD, \bU; \bY)$, the computation is efficient. More importantly, $\bu_t^*$ only depends on $\by_t$ rather than the entire atom dictionary. We will also show that \eqref{eq:u_t} suffices to guarantee the convergence of the algorithm to a stationary point of the expected loss function.

To gain intuition on the connection between $\tl_2(\bD, \bu; \bM, \by)$ and $\th(\bD, \bU; \bY)$, assume we have $n$ atoms in total (i.e., $\bY$ has $n$ columns). It turns out that for the function $\th(\bD, \bU; \bY)$, $\bU$ can be optimized using block coordinate minimization~(BCM), with each column of $\bU$ being a block variable. Initially, we may set all these columns to a zero vector. As BCM proceeds, we update $\bu_t$ while keeping the other $\bu_i$'s for $i \neq t$. Note that in this way, optimizing $\th(\bD, \bU; \bY)$ over $\bu_t$ amounts to minimizing the function $\tl_2(\bD, \bu; \bM_{t-1}, \by_t)$ where $M_{t-1}$ is defined in~\eqref{eq:M_t}. So after revealing all the atoms, each $\bu_t$ is sequentially updated only once. Henceforth, our strategy can be seen as a one-pass BCM algorithm for the function $\th(\bD, \bU; \bY)$.

\begin{algorithm}[t]
\caption{Online Low-Rank Subspace Clustering}
\label{alg:all}
\begin{algorithmic}[1]
    \REQUIRE $\bZ \in \Rpn$ (observed samples), $\bY \in \Rpn$, parameters $\lambda_1$, $\lambda_2$ and $\lambda_3$, random matrix $\bD_0 \in \Rpd$ (initial basis), zero matrices $\bM_0$, $A_0$ and $B_0$.
    \ENSURE Optimal basis $\bD_n$.
    \FOR{$t=1$ to $n$}
      \STATE Access the $t$-th sample $\bz_t$ and the $t$-th atom $\by_t$.

      \STATE Compute the coefficient and noise:
        \begin{align*}
         & \{\bv_t^*, \be_t^*\} = \argmin_{\bv, \be} \tl(\bD_{t-1}, \bv, \be; \bz_t),\\
          &\bu_t^* = \argmin_{\bu} \tl_2(\bD_{t-1}, \bM_{t-1}, \bu; \by_t).
        \end{align*}

      \STATE Update the accumulation matrices:
          \begin{align*}
          \bM_t &\leftarrow\ \bM_{t-1} + \by_t (\bu_t^*)\trans,\\
           \bA_t &\leftarrow\ A_{t-1} + \bv_t^*(\bv_t^*)\trans,\\
           \bB_t &\leftarrow\ B_{t-1} + (\bz_t - \be_t^*)(\bv_t^*)\trans.
          \end{align*}

      \STATE Update the basis:
      	  \begin{align*}
      	 \bD_t = \argmin_{\bD } \frac{1}{t} \Bigg[ \frac{1}{2}\tr\(\bD\trans \bD(\lambda_1 \bA_t + \lambda_3 \bI_d)\)  - \tr\(\bD\trans (\lambda_1 \bB_t + \lambda_3 \bM_t)\) \Bigg].
      	  \end{align*}
    \ENDFOR
\end{algorithmic}
\end{algorithm}

\subsubsection{Optimize $\bD$}
As soon as $\{ \bv_i^*, \be_i^*, \bu_i^* \}_{i=1}^t$ are available, we can refine the initial guess $\bD_{t-1}$ by optimizing the surrogate function
\begin{align}
\label{eq:g_t(D)}
g_t(\bD) \defeq \frac{1}{t} \Big(\sum_{i=1}^{t} \tl(\bD, \bv_i^*, \be_i^*; \bz_i) + \sum_{i=1}^{t} \fractwo{1} \twonorm{\bu_i^*}^2  + \fractwo{\lambda_3} \fronorm{\bD - \bM_t}^2 \Big).
\end{align}
Let us look at the first term:
\begin{align}\label{eq:gt_tilde}
\tilde{g}_t(\bD) \defeq \frac{1}{t} \sum_{i=1}^{t} \tl(\bD, \bv_i^*, \be_i^*; \bz_i),
\end{align}
which is a surrogate of $\frac{1}{t} \sum_{i=1}^{t} \ell(\bD; \bz_i)$. Though the above objective function involves the past iterates $\{\bv_i^*, \be_i^*\}_{i=1}^t$ whose memory cost is proportional to the sample size, we show that the minimizer of~\eqref{eq:gt_tilde} can be computed by accessing only two accumulation matrices whose sizes are independent of the sample size. To see this, we may expand the function $\tl(\bD, \bv_i^*, \be_i^*; \bz_i)$ as follows:
\begin{align*}
\tl(\bD, \bv_i^*, \be_i^*; \bz_i) =&\ \fractwo{\lambda_1} \twonorm{ \bz_i - \bD \bv_i^* - \be_i^* }^2 + \fractwo{1}\twonorm{\bv_i^*}^2 + \lambda_2 \onenorm{\be_i^*}\\
=&\ \frac{\lambda_1}{2} \tr\(\bD\trans\bD\bv_i^*(\bv_i^*) \trans\) + \lambda_1 \tr\( \bD\trans (\be_i^* - \bz_i) (\bv_i^*)\trans \) \\
&\ + \frac{\lambda_1}{2} \twonorm{\bz_i - \be_i^*}^2 + \fractwo{1}\twonorm{\bv_i^*}^2 + \lambda_2 \onenorm{\be_i^*}.
\end{align*}
Since the variable here is $\bD$, it holds that the minimizer of~\eqref{eq:gt_tilde} is given by the following:
\begin{align}\label{eq:first part}
\min_{\bD}\ \frac{1}{t} \sum_{i=1}^{t} \Bigg[ \frac{\lambda_1}{2} \tr\( \bD\trans\bD \bv_i^*(\bv_i^*) \trans \) + \lambda_1 \tr\( \bD\trans (\be_i^* - \bz_i) (\bv_i^*)\trans \)\Bigg].
\end{align}
As the trace is a linear mapping, solving the above program only needs to record two accumulation matrices
\begin{align}\label{eq:AtBt}
\bA_t \defeq \sum_{i=1}^{t} \bv_i^* (\bv_i^*)\trans \in \Rdd,\quad \bB_t \defeq \sum_{i=1}^{t} (\bz_i - \be_i^*) (\bv_i^*)\trans \in \Rpd.
\end{align}
Using the result of~\eqref{eq:first part} and doing some calculation, we derive $\bD_t$ as follows:
\begin{align}\label{eq:D_t}
\bD_t =&\ \argmin_{\bD } \frac{1}{t} \Big[ \frac{1}{2}\tr\(\bD\trans \bD(\lambda_1 \bA_t + \lambda_3 \bI_d)\)  - \tr\(\bD\trans (\lambda_1 \bB_t + \lambda_3 \bM_t)\) \Big]\notag\\
=&\  (\lambda_1 \bB_t + \lambda_3 \bM_t)  (\lambda_1 \bA_t + \lambda_3 \bI_d)^{-1},
\end{align}
where $\bA_t$ and $\bB_t$ are given in~\eqref{eq:AtBt} and $\bM_{t}$ is defined in~\eqref{eq:M_t}. Numerically, we apply coordinate descent to solve the above program owing to its efficiency. See more details in Appendix~\ref{supp:sec:alg}.

\subsubsection{Memory Cost}
It is remarkable that the  memory cost of Algorithm~\ref{alg:all} is $O(pd)$. To see this, note that when solving $\bv_t^*$ and $\be_t^*$, we load $\bD_{t-1}$ and a sample $\bz_t$ into the memory, which costs $O(pd)$. To compute the optimal $\bu_t^*$, we need to access $\bD_{t-1}$ and $\bM_{t-1} \in \Rpd$. Although we aim to minimize \eqref{eq:g_t(D)}, which seems to require all the past information, we actually only need to record $\bA_t$, $\bB_t$ and $\bM_t$, whose sizes are at most $O(pd)$ (recall that $d < p$).

\subsubsection{Time Complexity}
In addition to the memory efficiency,  we further elaborate that the the computation in each iteration is cheap. To compute $\{\bv_t^*, \be_t^*\}$, one may utilize the block coordinate method in~\citet{richtarik2014iteration} which enjoys linear convergence due to strong convexity. One may also apply the stochastic variance reduced algorithms which also ensure a geometric rate of convergence~\citep{xiao2014proximal,saga}. The $\bu_t^*$ is obtained by simple matrix-vector multiplications, which costs $O(pd)$. It is easy to see the complexity of Step~4 is $O(pd)$ and that of Step~5 is $O(pd^2)$.

\subsubsection{A Fully Online Subspace Clustering Scheme}
Now we have provided a way to learn the low-rank representation matrix $\bX$ in an online manner. Usually, researchers in the literature will take an optional post-processing step to refine the segmentation accuracy, for example, applying spectral clustering~\citep{ng2002spectral} on the obtained representation matrix $\bX^*$. In this case, one has to collect all the $\bu_i^*$'s and $\bv_i^*$'s to compute $\bX^* = (\bU^*)\trans \bV^*$ which will again increase the memory cost to $O(n^2)$. Here, we suggest an alternative scheme which admits $O(kd)$ memory usage where $k$ is the number of subspaces. The idea is appealing to the well-known $k$-means clustering in place of spectral clustering. One notable advantage is that updating the $k$-means model can be implemented in an online manner. In fact, the online $k$-means algorithm can be easily integrated into Algorithm~\ref{alg:all} by observing that $\bv_i^*$ is the robust feature for the $i$th sample. Second, we observe that $\bv_i^*$ is actually a robust feature for the $i$th sample. On the other hand, updating the $k$-means model is quite cheap, since the computational cost is $O(kd)$.

\subsubsection{An Alternative Online Implementation}
Our strategy for solving $\bu_t$ is based on a carefully designed proxy function which resolves Issue~\ref{is:partialY} and has a low complexity. Yet, to tackle Issue~\ref{is:partialY}, another potential way is to avoid the variable $\bu_t$\footnote{We'd like to thank the anonymous NIPS 2015 reviewer for pointing out this potential solution to the online algorithm.}. Recall that we derive the optimal solution $\tilde{\bU}$ (provided that $\bD$ is given) to $\th(\bD, \bU; \bY)$ as follows (see Prop~\ref{prop:opt U}):
\begin{align}\notag
	\tilde{\bU} = \bY\trans \( {\lambda_3^{-1}}\bI_p + \bY \bY\trans \)^{-1}\bD.
\end{align}
Plugging it back to $\th(\bD, \bU; \bY)$, we obtain
\begin{align*}
\th(\bD, \bU; \bY) = \frac{1}{2} \tr\( \bD \bD\trans \( Q_n - {\lambda_3}^{-1} Q^2_n \) \)  + \frac{\lambda_3}{2}  \fronorm{\bD - {\lambda_3}^{-1} Q_n\bD}^2,
\end{align*}
where
\begin{align*}
Q_n = \( {\lambda_3}^{-1} \bI_p + \bY \bY\trans \)^{-1}.
\end{align*}
Here, the subscript of $Q_n$ denotes the number of atoms in $\bY$. Note that the size of $Q_n$ is $p\times p$. Hence, if we incrementally compute the accumulation matrix $\bY \bY\trans = \sum_{i=1}^{t} \by_i \by_i\trans$, we can update the variable $D$ in an online fashion. Namely, at $t$-th iteration, we re-define the surrogate function as follows:
\begin{align*}
g_t(D) \defeq&\ \frac{1}{t} \Bigg[ \sum_{i=1}^{t} \tl(\bz_i, \bD, \bv_i, \be_i) + \fractwo{\lambda_3}  \fronorm{\bD -\frac{1}{\lambda_3} Q_t\bD}^2  \\
&\ + \fractwo{1} \tr\( \bD \bD\trans \( Q_t - \frac{1}{\lambda_3}Q^2_t \) \) \Bigg].
\end{align*}
Again, by noting the fact that $\tl(\bz_i, \bD, \bv_i, \be_i)$ only involves recording $\bA_t$ and $\bB_t$, we show that the memory cost is independent of sample size. However, the main shortcoming is that the time complexity of computing the inverse of a $p\times p$ matrix in each iteration is $O(p^3)$ which is not efficient in the high-dimensional regime (the time complexity of ours is proportional to $pd^2$). Either, it is not clear how to analyze the convergence.

\section{Theoretical Analysis}\label{sec:main results}

In this section, we present theoretical evidence that our algorithm guarantees that the sequence $\{ \bD_t \}_{t\geq 1}$ converges to a stationary point of the expected loss~\eqref{eq:f_n(D)} asymptotically. We note that the problem studied here is non-convex, and hence convergence to a stationary point is the best one can hope in general~\citep{bertsekas1999nonlinear}. As we have mentioned in related work, there are elegant results showing that convergence to global optimum is possible for batch alternating minimization under more stringent conditions on the data. See, for example,~\citet{jain2013low,jain2014nonconvexrpca,ge2016matrix}.

We make two assumptions throughout our analysis.
\begin{enumerate}[label=$(A\arabic*)$]
\item The observed data are generated i.i.d.\ from some (unknown) distribution and there exist constants $\alpha_0$ and $\alpha_1$, such that the conditions $0 < \alpha_0 \leq \twonorm{\bz_t} \leq \alpha_1$ and $ \alpha_0 \leq \twonorm{\by_t} \leq \alpha_1$ hold almost surely for all $t \geq 1$.\label{as:z}

\item The smallest singular value of $\frac{1}{t} \bA_t$ is bounded away from zero almost surely.\label{as:g_t(D)}
\end{enumerate}
Note that the first assumption is very mild. In fact, in many applications the data points are normalized to unit length, and in this case $\alpha_0 = \alpha_1 = 1$. To understand the second assumption, recall that $\bD_t$ is computed in~\eqref{eq:D_t}. If $\lambda_3 = o(t)$, then asymptotically the solution $\bD_t$ is not unique. Hence, Assumption~\ref{as:g_t(D)} ensures that we always have a unique solution that minimizes the surrogate function $g_t(\bD)$. Geometrically, as we have noted that $\bv_i$ is the coefficient of $\bz_i$ in terms of the basis dictionary $\bD$, the assumption simply requires that the data points are in general positions.

We also need the following parameter settings.
\begin{enumerate}[label=$(P\arabic*)$]
\item $\lambda_1$ and $\lambda_2$ are independent of the sample size $n$.\label{as:lambda12}
\item $\lim_{n \rightarrow \infty} \lambda_3/n = 0$.\label{as:lambda3}
\end{enumerate}
As we have shown in Section~\ref{sec:setup}, these parameter scalings give the expected loss function as in~\eqref{eq:f(D)}. We note that \ref{as:lambda12} is also required in previous work~\citep{mairal2010online,feng2013online,shen2014online}, though not set out explicitly. \ref{as:lambda3} is specific to our problem. It facilitates the characterization of the empirical loss function when $n$ tends to infinity. Otherwise, even the convergence of the sequence $\{ f_n(\bD) \}_{n\geq 1}$ is not clear (consider, e.g., $\lambda_3 = n \sin n$).

We present a fundamental result that will be heavily invoked in the subsequent analysis.
\begin{proposition}
\label{prop:bound:veABMDu}
Let $\{\bu_t^*\}_{t\geq 1}$, $\{\bv_t^*\}_{t\geq 1}$, $\{\be_t^*\}_{t\geq 1}$ and $\{\bD_t\}_{t\geq 1}$ be the sequence of the optimal solutions produced by Algorithm~\ref{alg:all}. Assume~\ref{as:z} and~\ref{as:g_t(D)}. Further suppose that $\lambda_2$ does not grow with $t$. Then, for all $t \geq 1$,

\begin{enumerate}
\item $\bv_t^*$, $\be_t^*$, $\frac{1}{t}\bA_t$ and $\frac{1}{t}\bB_t$  are uniformly bounded from above;

\item $\bM_t$ is uniformly upper bounded;

\item $\bD_t$ is supported on some compact set $\mathcal{D}$;

\item $\bu_t^*$ is uniformly upper bounded.
\end{enumerate}
\end{proposition}
Note that we prove the result by assuming a weaker condition than~\ref{as:lambda12}: the parameter $\lambda_2$ does not grow with the sample size $t$. While it is not surprising to see that $\bu_t^*, \bv_t^*, \be_t^*$ and $\bD_t$ are bounded from above due to the regularization, it is interesting to note that $\bM_t = \sum_{i=1}^{t} \by_i (\bu_i^*)\trans$ is also upper bounded. Intuitively, this holds since the last term in~\eqref{eq:reg batch} imposes that $\bM_t$ cannot deviate far from $\bD_t$. However, the technical challenge is that $\bD_t$ itself depends on $\bM_t$, as shown in~\eqref{eq:D_t}, and neither $\bD_t$ nor $\bM_t$ is bounded without the boundedness of the other. To remedy this, we propose a novel technique which conducts mathematical induction simultaneously on $\bD_t$ and $\bM_t$. This tremendously simplifies the proof of our earlier version~\citep{shen2016online}. See Appendix~\ref{app:proof:prop:bound} for the proof.

The proposition has many implications. For example, the uniform boundedness of the solutions immediately implies that the surrogate function $g_t(\bD)$ and the empirical loss function $f_t(\bD)$ are bounded from above for all $\bD \in \mathcal{D}$, which is necessary for the convergence.

Next, we prove that the sequence of $\{ g_t(\bD_t) \}_{t \geq 1}$ converges almost surely. We will make use of the following lemma, which is due to~\citet{bottou1998online}.
\begin{lemma}\label{lem:bottou}
Let $\( \Omega, \mathcal{F}, P \)$ be a measurable probability space, $\psi_t$, for $t \geq 1$, be the realization of a stochastic process and $\mathcal{F}_t$ be the filtration by the past information at time $t$. Let
\begin{align*}
\delta_t =
\begin{cases}
1\quad \text{if}\ \EXP[\psi_{t+1}-\psi_t \mid \mathcal{F}_t] > 0,\\
0\quad \text{otherwise}.
\end{cases}
\end{align*}
If for all $t$, $\psi_t \geq 0 $ and $\sum_{t=1}^{\infty} \EXP[\delta_t(\psi_{t+1}-\psi_t)] < \infty$, then $\psi_t$ is a quasi-martingale and converges almost surely. Moreover, it holds almost surely that
\begin{align*}
\sum_{t=1}^{\infty} \abs{ \EXP[\psi_{t+1} - \psi_t \mid \mathcal{F}_t] } < +\infty.
\end{align*}
\end{lemma}
Based on the lemma, we show the following result.
\begin{theorem}\label{thm:convergence g_t(D_t)}
Assume \ref{as:z} and~\ref{as:g_t(D)}. Set the parameters $\lambda_1$, $\lambda_2$ and $\lambda_3$ such that they satisfy~\ref{as:lambda12} and~\ref{as:lambda3}. Then the sequence of $\{g_t(\bD_t)\}_{t\geq 1}$ converges almost surely, where $\{\bD_t\}_{t\geq 1}$ is the solution produced by Algorithm~\ref{alg:all}. Moreover, the following holds almost surely:
\begin{equation}\label{eq:pf:sum diff u_t}
\sum_{t=1}^{\infty} \abs{ \EXP\big[\psi_{t+1} - \psi_t \mid \mathcal{F}_t\big] } < +\infty,
\end{equation}
where $\mathcal{F}_t = \{\bz_i, \by_i\}_{i=1}^t$.
\end{theorem}
\begin{proof}
(Sketch) We will view $\{ g_t(\bD_t) \}_{t\geq 1}$ as a non-negative stochastic process, with $\mathcal{F}_t = \{ \bz_i, \by_i \}_{i=1}^t$ be the filtration of the past information. In order to apply Lemma~\ref{lem:bottou}, we write $\psi_t := g_t(\bD_t)$ and compute the variation between two consecutive iterations:
\begin{align*}
 \psi_{t+1} - \psi_t =&\ \underbrace{g_{t+1}(\bD_{t+1}) - g_{t+1}(\bD_t)}_{\zeta_1} + \underbrace{ \frac{\tilde{f}_t(\bD_t) - \tilde{g}_t(\bD_t)}{t+1} }_{\zeta_2} + \underbrace{ \frac{\ell(\bz_{t+1}, \bD_t) - \tilde{f}_t(\bD_t)}{t+1} }_{\zeta_3} \\
&\ + \Bigg[ \frac{1}{t+1} \sum_{i=1}^{t+1} \fractwo{1} \twonorm{ \bu_i^* }^2 + \frac{\lambda_3}{2(t+1)} \fronorm{\bD_t - \bM_{t+1}}^2 \notag\\
&\ - \frac{1}{t} \sum_{i=1}^{t} \frac{1}{2} \twonorm{ \bu_i^* }^2 - \frac{\lambda_3}{2t} \fronorm{\bD_t - \bM_{t}}^2 \Bigg].
\end{align*}
Here,
\begin{align*}
\tilde{f}_t(\bD_t) = \frac{1}{t} \sum_{i=1}^{t} \ell(\bD_t; \bz_i).
\end{align*}
Since $\bD_{t+1}$ minimizes $g_{t+1}(\bD)$, we have $\zeta_1 \leq 0$. Also, recall that $\tilde{g}_t(\bD)$ is a surrogate of $\tilde{f}_t(\bD)$, implying $\zeta_2 \leq 0$. For $\zeta_3$, the P-Donsker lemma (see Prop.~\ref{prop:l donsker}) gives
\begin{align*}
\EXP\big[ \EXP[ \zeta_3 \mid \mathcal{F}_t ] \big] \leq \frac{\const}{\sqrt{t}(t+1)}
\end{align*}
for some constant $\const$. Finally, we need to upper bound the terms in the brackets which are specific to the online subspace clustering problem. In light of the bound of $\zeta_3$, one may look for a way to bound them with $O(1/t^{3/2})$. Surprisingly, it turns out that our algorithm automatically guarantees that the sum of the terms in the brackets is not greater than zero. This follows by noting the closed-form solution of $\bu_i^*$~\eqref{eq:u_t} and the fact that $\bM_t = \sum_{i=1}^{t} \by_i (\bu_i^*)\trans$. Putting all the pieces together completes the proof.
\end{proof}

We move on to show that the sequence of $\{ f_t(\bD_t) \}_{t\geq 1}$ converges. In particular, we justify that $g_t(\bD)$ acts as a good surrogate of $f_t(\bD)$ in the sense that the sequence of the surrogate converges to the same limit of that of the empirical loss.
\begin{theorem}\label{thm:convergence f_t(D_t)}
Assume \ref{as:z} and~\ref{as:g_t(D)}. Set the parameters $\lambda_1$, $\lambda_2$ and $\lambda_3$ such that they satisfy~\ref{as:lambda12} and~\ref{as:lambda3}. Let $\{\bD_t\}_{t\geq 1}$ be the solution produced by Algorithm~\ref{alg:all}. Then, the sequence of $\{f_t(\bD_t)\}_{t\geq 1}$ converges almost surely to the same limit of $\{g_t(\bD_t)\}_{t\geq 1}$.
\end{theorem}
\begin{corollary}\label{cor:diff f}
Assume sames conditions as in Theorem~\ref{thm:convergence f_t(D_t)}. Then the sequence of $\{f(\bD_t)\}_{t\geq 1}$ converges almost surely to the same limit of $\{f_t(\bD_t)\}_{t\geq 1}$ or equivalently, $\{g_t(\bD_t)\}_{t\geq 1}$.
\end{corollary}

\begin{proof}
(Sketch) We need several tools in the literature to prove this result. From a high level, since we have already shown that $\{ g_t(\bD_t) \}_{t\geq 1}$ converges, we only need to deduce that the limit of $g_t(\bD_t) - f_t(\bD_t)$ is zero. A useful result for this purpose is stated below, which is borrowed from~\citet{mairal2010online}.
\begin{lemma}[Lemma 8 in~\citet{mairal2010online}]\label{lem:mairal}
Let $\{a_t\}_{t\geq 1}$, $\{b_t\}_{t\geq 1}$ be two real sequences such that for all $t$, $a_t \geq 0$, $b_t \geq 0$, $\sum_{t=1}^{\infty}a_t = \infty$, $\sum_{t=1}^{\infty} a_t b_t < \infty$, there exists a scalar $\const > 0$, such that $\abs{ b_{t+1} - b_t } < \const \cdot a_t$. Then, $\lim_{t \rightarrow \infty} b_t = 0$.
\end{lemma}
In particular, we set $a_t = (t+1)^{-1}$ and $b_t = g_t(\bD_t) - f_t(\bD_t) \geq 0$. First, we verify that the sum of the infinite series $\{ a_t b_t \}_{t \geq 1}$ is finite. By algebra, we obtain
\begin{align*}
\frac{b_t}{t+1} \leq \frac{\tilde{g}_t(\bD_t) - \tilde{f}_t(\bD_t)}{t+1} + \frac{q_t(\bD_t)}{t+1},
\end{align*}
where
\begin{align*}
q_t(\bD_t) = \frac{1}{t} \sum_{i=1}^{t} \fractwo{1} \twonorm{\bu_i^*}^2 + \frac{\lambda_3}{2t} \fronorm{\bD_t - \bM_t}^2.
\end{align*}
Then by utilizing the uniform boundedness of $\bD_t$ and $\bM_t$, it is possible to derive
\begin{align*}
 \frac{b_t}{t+1} \leq \frac{ \ell(\bD_t; \bz_{t+1}) - \tilde{f}_t(\bD_t) }{t+1} + \psi_t - \psi_{t+1} + \frac{ \const}{2t(t+1)}
\end{align*}
for some absolute constant $\const$. As we have shown in the proof sketch of Theorem~\ref{thm:convergence g_t(D_t)}, the first term scales as $O(1/t^{3/2})$, which combined with \eqref{eq:pf:sum diff u_t} imply that $\sum_{t=1}^{\infty} \frac{b_t}{t+1}$ is finite.

Next, we claim that
\begin{align*}
\abs{b_{t+1} - b_t} \leq \frac{\const_0}{t+1},
\end{align*}
for some absolute constant $\const_0$. The follows heavily from Prop.~\ref{prop:bound:veABMDu} where we showed that all the variables are uniformly bounded, and hence all involved functions evaluated at these solutions are bounded from above. To be more concrete, using triangle inequality we get
\begin{align*}
\abs{ b_{t+1} - b_t } \leq&\  \abs{ g_{t+1}(\bD_{t+1}) - g_t(\bD_{t+1}) } +\abs{ f_{t+1}(\bD_{t+1}) - f_t(\bD_{t+1}) } \\
&\ +  \abs{ g_t(\bD_{t+1}) - g_t(\bD_t) } + \abs{ f_t(\bD_{t+1}) - f_t(\bD_t) }.
\end{align*}
The first two terms on the right-hand side are upper bounded by $O(1/t)$ due to uniform boundedness. For the last two terms, we utilize the fact that $g_t(\bD)$ and $f_t(\bD)$ are both Lipschitz to show that they are bounded by $O(\fronorm{\bD_{t+1} - \bD_t})$ from above. The following proposition illustrates that the variation of $\bD_{t+1}$ and $\bD_t$ vanishes with the rate $O(1/t)$, whose proof can be found in Appendix~\ref{app:proof:prop:diff_D}.
\begin{proposition}\label{prop:diff_D}
Assume~\ref{as:z} and~\ref{as:g_t(D)}. Further suppose that $\lambda_1$ and $\lambda_2$ do not grow with $t$. Let $\{\bD_t\}_{t\geq 1}$ be the basis sequence produced by Algorithm~\ref{alg:all}. Then,
\begin{equation*}
\fronorm{\bD_{t+1} - \bD_t} = O\(\frac{1}{t}\).
\end{equation*}
\end{proposition}
Thus, in allusion to Lemma~\ref{lem:mairal}, we complete the proof of Theorem~\ref{thm:convergence f_t(D_t)}. The corollary follows immediately because the central limit theorem asserts that $\sqrt{t}(f(\bD_t) - f_t(\bD_t))$ is upper bounded.
\end{proof}

Finally, we show that asymptotically, $\bD_t$ acts as a stationary point of the expected loss function.
\begin{theorem}\label{thm:stationary}
Assume~\ref{as:z},~\ref{as:g_t(D)},~\ref{as:lambda12} and~\ref{as:lambda3}. Let $\{\bD_t\}_{t=1}^{\infty}$ be the sequence of optimal bases produced by Algorithm~\ref{alg:all}. Then, the sequence converges to a stationary point of the expected loss function $f(\bD)$ when $t$ goes to infinity.
\end{theorem}

\section{Experiments}\label{sec:exp}

This section gives empirical evidence that our algorithm OLRSC is fast and robust. We demonstrate by simulations that the solution is good enough. In fact, we find that after revealing all the data points, OLRSC is able to recover the true subspace even with grossly corrupted entries. Thus, it is interesting to study the statistical performance of OLRSC in the future work. We also illustrate that OLRSC is orders of magnitude faster than batch methods such as LRR and SSC.

\subsection{Settings}
Before presenting the empirical results, we first introduce the universal settings used throughout the section.

\subsubsection{Baselines} 
For the subspace recovery task, we compare our algorithm with ORPCA~\citep{feng2013online}, LRR~\citep{liu2013robust} and PCP~\citep{candes2011robust}. For the subspace clustering task, we choose ORPCA, LRR and SSC~\citep{elhamifar2009sparse} as the competitive baselines. Recently, \citet{liu2014recovery} improved the vanilla LRR by utilizing some low-rank matrix for $\bY$. We denote this variant of LRR by LRR2 and accordingly, our algorithm equipped with such a basis dictionary $\bY$ is denoted as OLRSC2.

\subsubsection{Evaluation Metric} 
We evaluate the fitness of the recovered subspaces $\bD$ (with each column being normalized) and the ground truth $L$ by the Expressed Variance (EV)~\citep{xu2010principal}:
\begin{align}
\text{EV}(\bD, \bL) \defeq \frac{\tr(DD\trans \bL\bL\trans)}{\tr(\bL\bL\trans)}.
\end{align}
The value of EV is between 0 and 1, and a higher value means better recovery (EV $=1$ means exact recovery).

The performance of subspace clustering is measured by clustering accuracy which is provided in the SSC toolkit. Its value also ranges in the interval $[0, 1]$, and a higher value indicates a more accurate clustering.

\subsubsection{Parameters} 
We set $\lambda_1 = 1$, $\lambda_2 = {1}/{\sqrt{p}}$ and $\lambda_3 = \sqrt{t/p}$, where $t$ is the iteration counter. Note that the parameter settings satisfy~\ref{as:lambda12} and~\ref{as:lambda3}. In particular, $\lim_{t \rightarrow \infty} \lambda_3/t = 0$. We follow the default parameter setting for the baselines.

\subsection{Subspace Recovery}\label{subsec:recovery}

\subsubsection{Simulation Data}
We use 4 disjoint subspaces $\{\mathcal{S}_k \}_{k=1}^4 \subset \Rp$, whose bases are denoted by $\{ L_k \}_{k=1}^4 \in \mathbb{R}^{p \times d_k}$. The clean data matrix $\bar{\bZ}_k \in \mathcal{S}_k$ is then produced by $\bar{\bZ}_k = L_k R_k\trans$, where $R_k \in \mathbb{R}^{n_k \times d_k}$. The entries of $L_k$'s and $R_k$'s are sampled i.i.d.\  from the normal distribution. Finally, the observed data matrix $\bZ$ is generated by $\bZ = \bar{\bZ} + \bE$, where $\bar{\bZ}$ is the column-wise concatenation of $\bar{\bZ}_k$'s followed by a random permutation, $\bE$ is the sparse corruption whose $\rho$ fraction entries are non-zero and follow an i.i.d.\ uniform distribution over $[-2, 2]$. We independently  conduct each experiment 10 times and report the averaged results.

\subsubsection{Robustness}
We illustrate by simulation results that OLRSC can effectively recover the underlying subspaces, confirming that $\bD_t$ converges to the union of subspaces. For the two online algorithms OLRSC and ORPCA, We compute the EV after revealing all the samples. We examine the performance under different intrinsic dimension $d_k$'s and corruption $\rho$. To be more detailed, the $d_k$'s are varied from $0.01p$ to $0.1p$ with a step size $0.01p$, and the $\rho$ is from 0 to 0.5, with a step size 0.05.

\begin{figure}[t]
\centering
\includegraphics[width=0.45\linewidth]{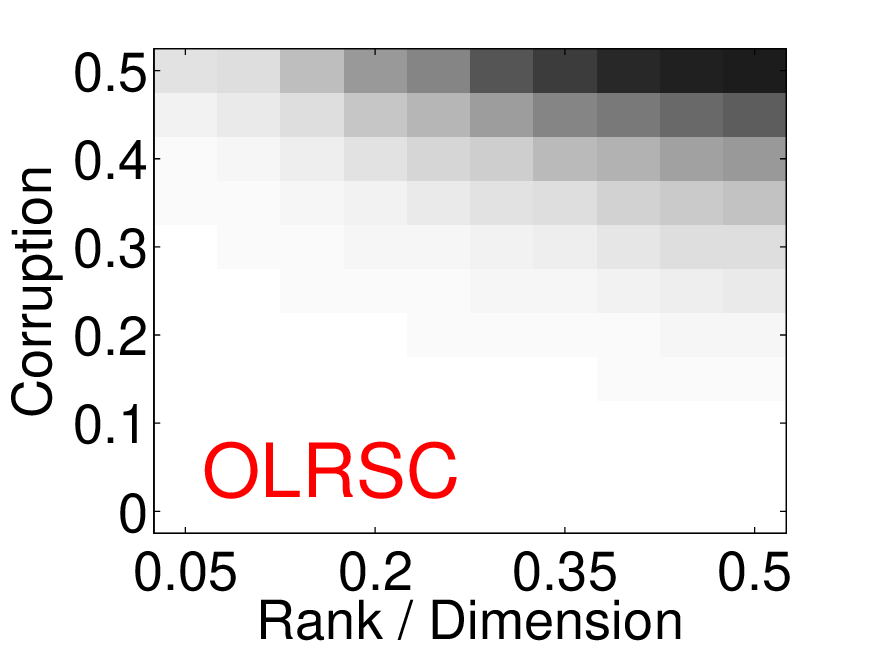}
\includegraphics[width=0.45\linewidth]{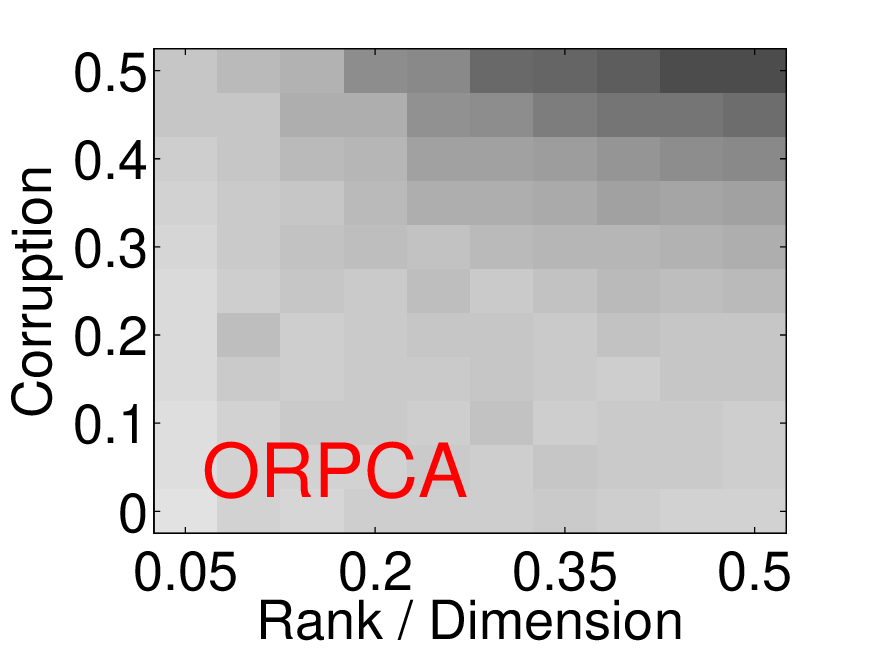}

\includegraphics[width=0.45\linewidth]{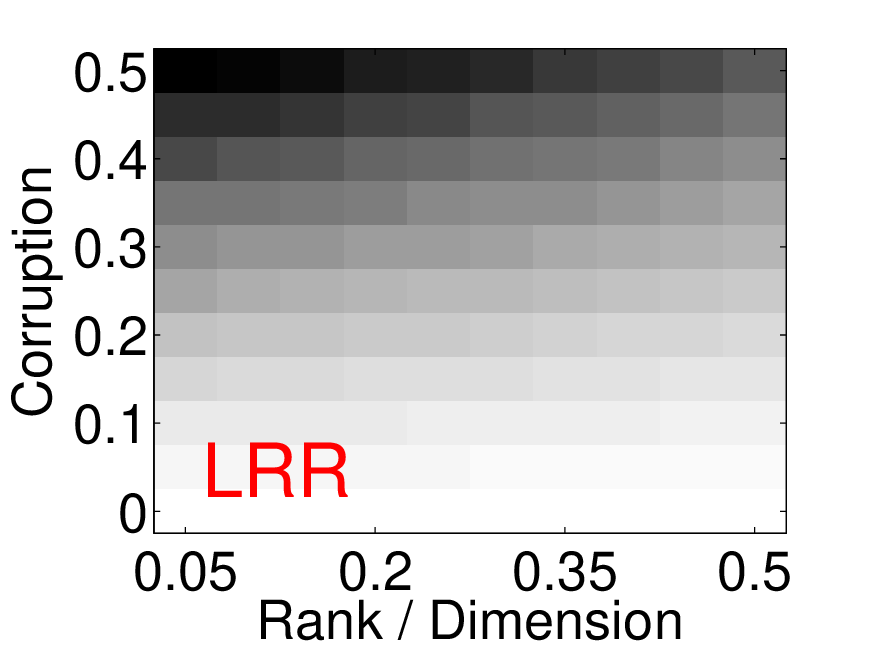}
\includegraphics[width=0.45\linewidth]{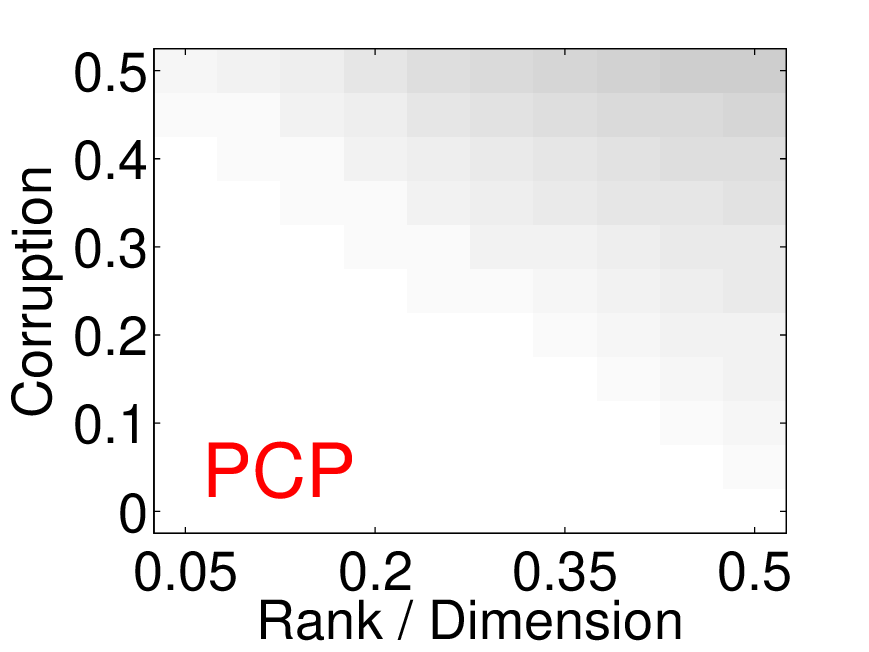}

\caption{{\bf Subspace recovery under different intrinsic dimensions and corruptions.} Brighter is better. We set $p = 100$, $n_k = 1000$ and $d = 4d_k$. LRR and PCP are batch methods. OLRSC consistently outperforms ORPCA and even improves the performance of LRR. Compared to PCP, OLRSC is competitive in most cases and degrades a little for highly corrupted data, possibly due to the number of samples is not sufficient for its convergence.
}
\label{fig:diff_rank_rho}
\end{figure}

The results are presented in Figure~\ref{fig:diff_rank_rho}. The most intriguing observation is that OLRSC as an online algorithm outperforms its batch counterpart LRR! Such improvement may come from the explicit modeling for the basis, which makes OLRSC more informative than LRR. To fully understand the rationale behind this phenomenon is an important direction for future research. Notably, OLRSC consistently beats ORPCA (an online version of PCP), in that OLRSC takes into account that the data are produced by a union of small subspaces. While PCP works well for almost all scenarios, OLRSC degrades a little when addressing difficult cases (high rank and corruption). This is not surprising since Theorem~\ref{thm:stationary} is based on asymptotic analysis and hence, we expect that OLRSC will converge to the true subspace after acquiring more samples.

\subsubsection{Convergence Rate and Time Complexity}
Now we test on a large dataset to show that our algorithm usually converges to the true subspace faster than ORPCA. We plot the EV curve against the number of samples in Figure~\ref{fig:ev_samples}. Firstly, when equipped with a proper matrix $\bY$, OLRSC2 and LRR2 can always produce an exact recovery of the subspace as PCP does. When using the dataset itself for $\bY$, OLRSC still converges to a favorable point after revealing all the samples. Compared to ORPCA, OLRSC is more robust and converges much faster for hard cases (see, e.g., $\rho = 0.5$). Again, we note that in such hard cases, OLRSC outperforms LRR, which agrees with the observation in Figure~\ref{fig:diff_rank_rho}.

\begin{figure}[t]
\centering
{\includegraphics[width=0.45\linewidth]{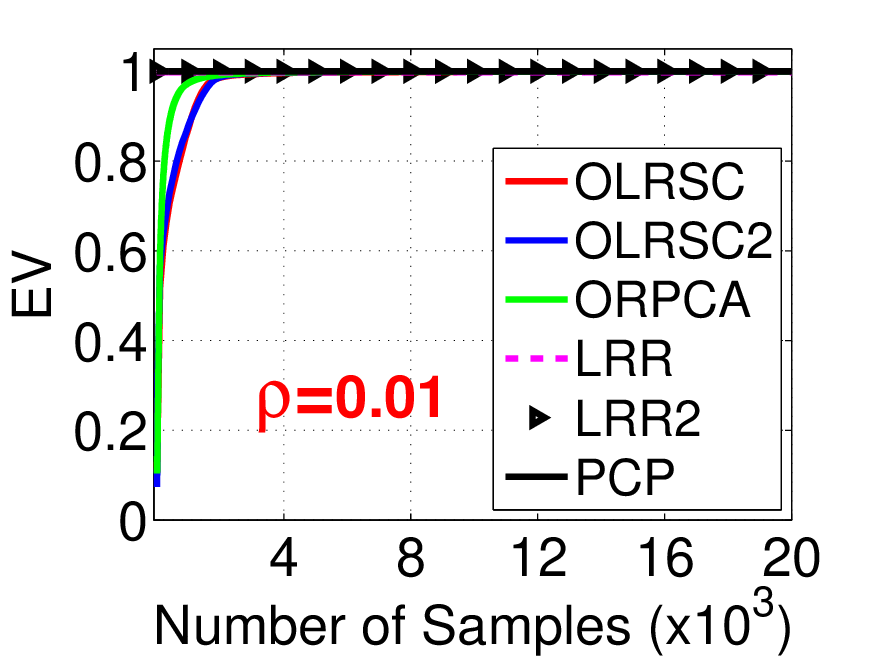}}
{\includegraphics[width=0.45\linewidth]{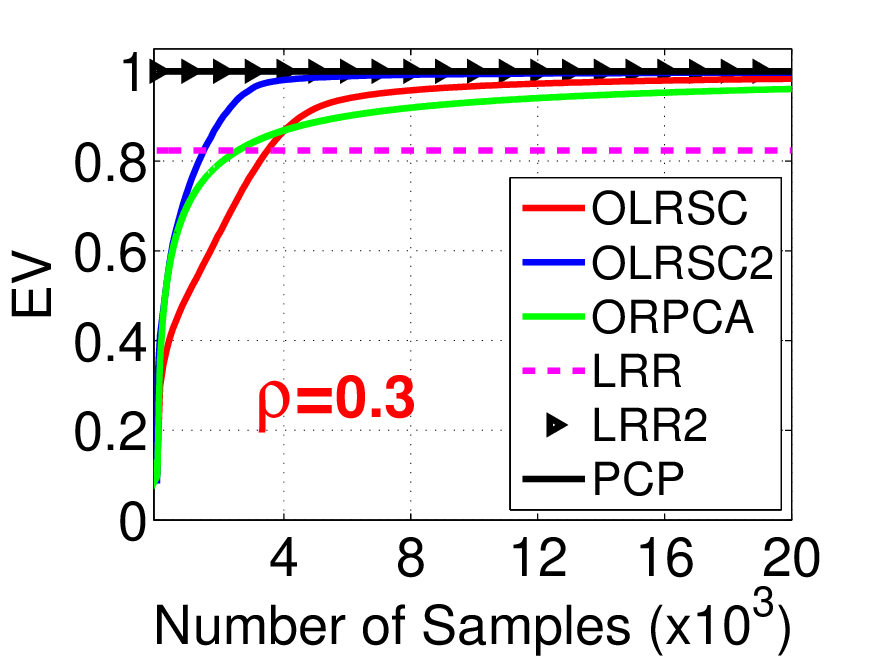}}

{\includegraphics[width=0.45\linewidth]{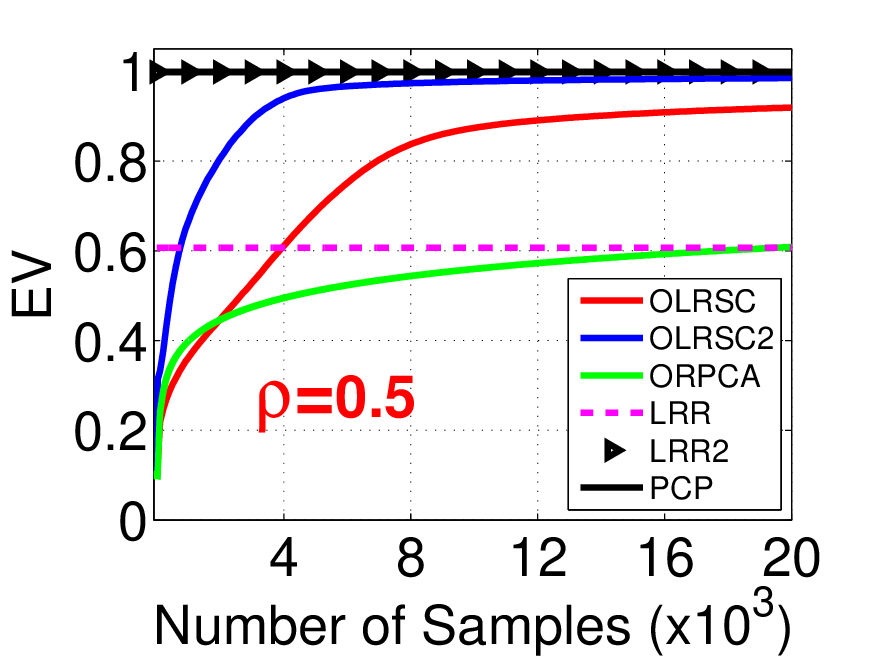}}
{\includegraphics[width=0.45\linewidth]{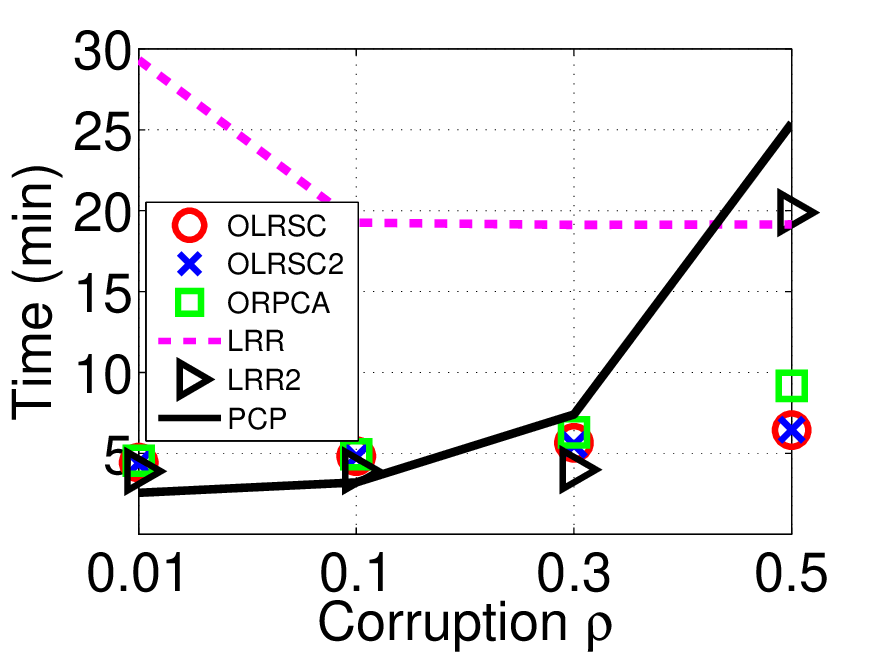}}
\caption{{\bf Convergence rate and time complexity.} A higher EV means better subspace recovery. We set $p = 1000$, $n_k = 5000$, $d_k = 25$ and $d = 100$. OLRSC always converges to or outperforms the batch counterpart LRR. For hard cases, OLRSC converges much faster than ORPCA. Both PCP and LRR2 achieve the best EV value. When equipped with the same dictionary as LRR2, OLRSC2 also well handles the highly corrupted data ($\rho=0.5$). Our methods are more efficient than the competitors but PCP when $\rho$ is small, possibly because PCP utilizes a highly optimized C++ toolkit while ours are written in Matlab.
}
\label{fig:ev_samples}
\end{figure}

We also illustrate the time complexity of the algorithms in the last panel of Figure~\ref{fig:ev_samples}. In short, our algorithms (OLRSC and OLRSC2) admit the lowest computational complexity for all cases. One may argue that PCP spends slightly less time than ours for a small $\rho$ (0.01 and 0.1). However, we remark here that PCP utilizes a highly optimized C++ toolkit to boost computation while our algorithms are {fully} written in Matlab. We believe that ours will work more efficiently if properly optimized by, {e.g.,} the blas routine. Another important message conveyed by the figure is that, OLRSC is always being orders of magnitude computationally more efficient than the batch method LRR, as well as producing comparable or even better solution.

\subsection{Subspace Clustering}\label{subsec:clustering}

\subsubsection{Datasets}
We examine the performance of subspace clustering on 5 realistic databases shown in Table~\ref{tb:dataset}, which can be downloaded from the LibSVM website. For MNIST, We randomly select 20 thousands samples to form MNIST-20K since we find it time consuming to run the batch methods on the entire database.

\begin{table}[h]
\vspace{-0.13in}
	\centering
	\caption{\bf Datasets for subspace clustering.}
	\vspace{0.05in}
	\begin{tabular}{lccc}
		\toprule
		& \#classes & \#samples & \#features \\
		\midrule
		Mushrooms & 2         & 8124      & 112        \\
		DNA       & 3         & 3186      & 180        \\
		Protein   & 3         & 24,387    & 357        \\
		USPS      & 10        & 9298      & 256        \\
		MNIST-20K & 10        & 20,000    & 784       \\
		\bottomrule
	\end{tabular}
	\label{tb:dataset}
\vspace{-0.05in}
\end{table}

\subsubsection{Standard Clustering Pipeline}
In order to focus on the solution quality of different algorithms, we follow the standard pipeline which feeds $\bX$ to a spectral clustering algorithm~\citep{ng2002spectral}. To this end, we collect all the $\bu$'s and $\bv$'s produced by OLRSC to form the representation matrix $\bX=\bU\bV\trans$. For ORPCA, we use $R_0R_0\trans$ as the similarity matrix~\citep{liu2013robust}, where $R_0$ is the row space of $\bZ_0=L_0\Sigma_0R_0\trans$ and $\bZ_0$ is the clean matrix recovered by ORPCA. We run our algorithm and ORPCA with 2 epochs so as to refine the coefficients (i.e., $\bU$ and $\bV$ in ours and $R_0$ in ORPCA). Note that for subspace clustering, this step is essential because the initial guess of $\bD$ results in bad solutions of the coefficients at the beginning.

\subsubsection{Fully Online Pipeline}
As we discussed in Section~\ref{sec:alg}, the (optional) spectral clustering procedure needs the similarity matrix $\bX$, making the memory proportional to $n^2$. To tackle this issue, we proposed a fully online scheme where the key idea is performing $k$-means on $\bV$. Here, we examine the efficacy of this variant, which is called OLRSC-F.

\subsubsection{Results}
The results are recorded in Table~\ref{tb:sc}, where the time cost of spectral clustering or $k$-means is not included so we can focus on comparing the efficiency of the algorithms themselves. Also note that we use the dataset itself as the dictionary $Y$ because we find that an alternative choice of $Y$ does not help too much on this task. For OLRSC and ORPCA, they require an estimation on the true rank. Here, we use $5k$ as such estimation where $k$ is the number of classes of a dataset. Our algorithm significantly outperforms the two state-of-the-art methods LRR and SSC both in terms of accuracy and efficiency. One may argue that SSC is slightly better than OLRSC on Protein. Yet, it spends 1 hour while OLRSC only costs 25 seconds. Hence, SSC is not practical. Compared to ORPCA, OLRSC always identifies more correct samples as well as consumes comparable running time. For example, on the USPS dataset, OLRSC achieves the accuracy of 65.95\% while that of ORPCA is 55.7\%. Regarding the running time, OLRSC uses only 7 seconds more than ORPCA~--~same order of computational complexity, which agrees with the qualitative analysis in Section~\ref{sec:alg} and the one in~\cite{feng2013online}.

\begin{table}[t]
\caption{{\bf Clustering accuracy (\%) and computational time (seconds in default).} For each dataset, the first row indicates the accuracy and the second row the running time. For all the large-scale datasets, OLRSC (or OLRSC-F) has the highest clustering accuracy. Regarding the running time, our method spends comparable time as ORPCA~(the fastest solver) does while dramatically improves the accuracy. Although SSC is slightly better than SSC on Protein, it consumes one hour while OLRSC takes 25 seconds.
}
\centering
\begin{tabular}{lccccc}
\toprule
       & OLRSC &OLRSC-F  & ORPCA & LRR   & SSC \\
\midrule
Mush- & {85.09} & {\bf 89.36}  & 65.26 & 58.44  & 54.16 \\
rooms  & 8.78 & 8.78  & 8.30 & 46.82  & 32 min\\
\midrule
\multirow{2}{*}{DNA}  & {67.11} & {\bf 83.08}  & 53.11 & 44.01  & 52.23\\
 & 2.58 & 2.58 & 2.09 & 23.67 & 3 min\\
\midrule
\multirow{2}{*}{Protein}  & { 43.30} & 43.94  & 40.22 & 40.31  & {\bf 44.27}\\
 & 24.66 & 24.66 & 22.90 & 921.58 & 65 min\\
\midrule
\multirow{2}{*}{USPS}  & {65.95} & {\bf 70.29}  & 55.70 & 52.98  & 47.58\\
 & 33.93 & 33.93  & 27.01 & 257.25 & 50 min\\
\midrule
MNIST-  & {\bf 57.74} & 55.50  & 54.10 & 55.23 &  43.91\\
20K & 129 & 129  & 121 & 32 min &  7 hours\\
\bottomrule
\end{tabular}
\label{tb:sc}
\end{table}

More interestingly, it shows that the $k$-means alternative (OLRSC-F) usually outperforms the spectral clustering pipeline. This suggests that perhaps for {\em robust} subspace clustering formulations, the simple $k$-means paradigm suffices to guarantee an appealing result. On the other hand, we report the running time of spectral clustering and $k$-means in Table~\ref{tab:fully_online}. As expected, since spectral clustering computes SVD for an $n$-by-$n$ similarity matrix, it is quite slow. In fact, it sometimes dominates the running time of the whole pipeline. In contrast, $k$-means is extremely fast and scalable, as it can be implemented in online fashion.

\begin{table}[h]
\caption{{\bf Time cost (seconds) of spectral clustering and $k$-means.}}

\centering
\begin{tabular}{lccccc}
\toprule
       & Mushrooms & DNA & Protein & USPS & MNIST-20K\\
\midrule
Spectral & 295 & 18 & 7567 & 482 & 4402\\
$k$-means & 2 & 6 & 5 & 19 & 91\\
\bottomrule
\end{tabular}
\label{tab:fully_online}
\end{table}

\subsection{Influence of $d$}

A key ingredient of our formulation is a factorization on the nuclear norm regularized matrix, which requires an estimation on the rank of the $\bX$ (see \eqref{eq:nuclear reform}). Here we examine the influence of the choice of $d$ (which plays as an upper bound of the true rank). We report both EV and clustering accuracy for different $d$ under a range of corruptions. The simulation data are generated as in Section~\ref{subsec:recovery} and we set $p=200$, $n_k=1000$ and $d_k = 10$. Since the four subspaces are disjoint, the true rank is 40.

\begin{figure}[t]
\centering
	{\includegraphics[width=0.45\linewidth]{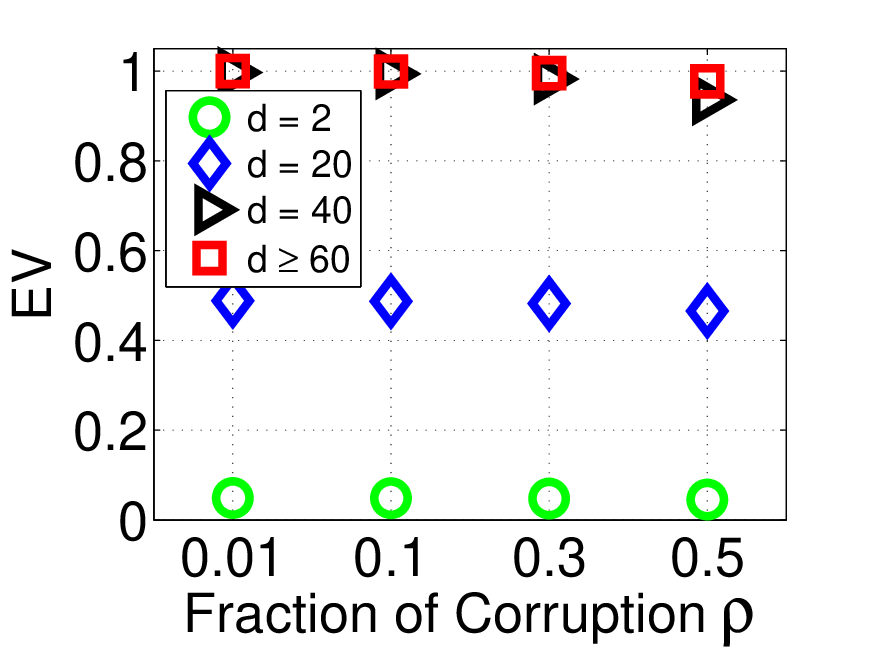}}
	{\includegraphics[width=0.45\linewidth]{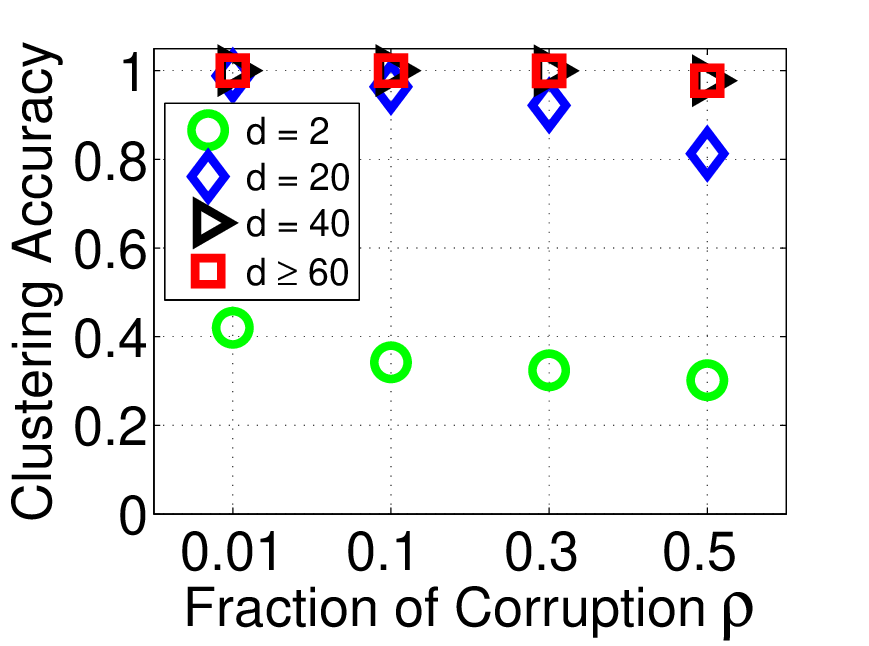}}

	\caption{{\bf Examine the influence of ${d}$.} The true rank is 40. We experiment on $d = \{ 2, 20, 40, 60, 80, 100, 120, 140, 160, 180 \}$.   }
	\label{fig:diff_d}
\end{figure}

From Figure~\ref{fig:diff_d}, we observe that our algorithm cannot recover the true subspace if $d$ is smaller than the true rank. On the other hand, when $d$ is sufficiently large (at least larger than the true rank), our algorithm can perfectly estimate the subspace. This agrees with the results in~\cite{burer2005local} which says as long as $d$ is large enough, any local minima is global optima. We also illustrate the influence of $d$ on subspace clustering. Generally speaking, OLRSC can consistently identify the cluster of the data points if $d$ is sufficiently large. Interestingly, different from the subspace recovery task, here the requirement for $d$ seems to be slightly relaxed. In particular, we notice that if we pick $d$ as 20 (smaller than the true rank), OLRSC still performs well. Such relaxed requirement of $d$ may benefit from the fact that the spectral clustering step can correct some wrong points as suggested by~\cite{soltanolkotabi2014robust}.


\section{Conclusion}\label{sec:conclusion}
In this paper, we have proposed an online algorithm termed OLRSC for subspace clustering, which dramatically reduces the memory cost of LRR from $O(n^2)$ to $O(pd)$. One of the key techniques is an explicit basis modeling, which essentially renders the model more informative than LRR. Another important component is a non-convex reformulation of the nuclear norm. Combining these techniques allows OLRSC to simultaneously recover the union of the subspaces, identify the possible corruptions and perform subspace clustering. We have also established the theoretical guarantee that solutions produced by our algorithm  converge to a stationary point of the expected loss function. Moreover, we have analyzed the time complexity and empirically demonstrated that our algorithm is computationally very efficient compared to competing baselines. Our extensive experimental study on synthetic and realistic datasets  also illustrates the robustness of OLRSC. In a nutshell, OLRSC is an appealing algorithm  in all three worlds: memory cost, computation and robustness.

\clearpage

\clearpage

\appendix

\section{Algorithm Details}
\label{supp:sec:alg}
\begin{algorithm}
\caption{Solving $\bv$ and $\be$}
\label{alg:ve}
\begin{algorithmic}[1]
    \REQUIRE $\bD \in \Rpd$, $\bz \in \Rp$, parameters $\lambda_1$ and $\lambda_2$
    \ENSURE Optimal $\bv$ and $\be$.
    \STATE Set $\be = \boldsymbol{0}$.
    \REPEAT
        \STATE Update $\bv$:
               \begin{align*}
               \bv = (\bD\trans \bD + \frac{1}{\lambda_1}\bI)^{-1}\bD\trans (\bz - \be).
               \end{align*}

        \STATE Update $\be$:
               \begin{align*}
               \be = \mathcal{S}_{\lambda_2 / \lambda_1}[\bz - \bD\bv].
               \end{align*}
    \UNTIL{convergence}
\end{algorithmic}
\end{algorithm}

\begin{algorithm}
\caption{Solving $\bD$}
\label{alg:D}
\begin{algorithmic}[1]
    \REQUIRE $\bD \in \Rpd$ in the previous iteration, accumulation matrix $\bM$, $\bA$ and $\bB$, parameters $\lambda_1$ and $\lambda_3$.
    \ENSURE Optimal $\bD$ (updated).
    \STATE Denote $\widehat{A} = \lambda_1 A + \lambda_3 \bI$ and $\widehat{B} = \lambda_1B + \lambda_3 \bM $.
    \REPEAT
      \FOR{$j = 1$ to $d$}
      \STATE Update the $j$th column of $\bD$:
      \begin{align*}
        \bd_j \leftarrow \bd_j - \frac{1}{\widehat{a}_{jj}} \( \bD \widehat{\ba}_j - \widehat{\bb}_j \)
      \end{align*}
      \ENDFOR
    \UNTIL{convergence}
\end{algorithmic}
\end{algorithm}
For Algorithm~\ref{alg:ve}, we set a threshold $\epsilon = 10^{-3}$. Let $\{\bv', \be'\}$ and $\{\bv'', \be''\}$ be the two consecutive iterates. If the maximum of ${\twonorm{\bv' - \bv''}}/{\twonorm{\bv'}}$ and ${\twonorm{\be' - \be''}}/{\twonorm{\be'}}$ is less than $\epsilon$, then we stop Algorithm~\ref{alg:ve}.

For Algorithm~\ref{alg:D}, we observe that a one-pass update on the dictionary $\bD$ is enough for the final convergence of $\bD$, as we showed in the experiments. This is also observed in~\citet{mairal2010online}.

\section{Proofs}\label{supp:sec:proof}

\subsection{Technical Lemmas}
We need several technical lemmas for our proof.
\begin{lemma}[Corollary of Thm. 4.1 in~\citet{bonnans1998optimization}]
\label{lem:bonnans}
Let $f: \Rp \times \Rq \rightarrow \mathbb{R}$. Suppose that for all $\bx \in \Rp$ the function $f(\bx, \cdot)$ is differentiable, and that $f$ and $\nabla_{\bu}f(\bx, \bu)$ are continuous on $\Rp \times \Rq$. Let $\bv(\bu)$ be the optimal value function $\bv(\bu) = \min_{\bx \in \mathcal{C}}f(\bx, \bu)$, where $\mathcal{C}$ is a compact subset of $\Rp$. Then $\bv(\bu)$ is directionally differentiable. Furthermore, if for $\bu_0 \in \Rq$, $f(\cdot, \bu_0)$ has unique minimizer $\bx_0$ then $\bv(\bu)$ is differentiable in $\bu_0$ and $\nabla_{\bu} \bv(\bu_0) = \nabla_{\bu} f(\bx_0, \bu_0)$.
\end{lemma}

\begin{lemma}[Corollary of Donsker theorem~\citep{van2000asymptotic}]
\label{lem:donsker}
Let $F = \{f_{\theta}: \mathcal{X} \rightarrow \mathbb{R}, \theta \in \Theta\}$ be a set of measurable functions indexed by a bounded subset $\Theta$ of $\mathbb{R}^d$. Suppose that there exists a constant $K$ such that
\begin{align*}
\abs{ f_{\theta_1}(x) - f_{\theta_2}(x) } \leq K \twonorm{\theta_1 - \theta_2},
\end{align*}
for every $\theta_1$ and $\theta_2$ in $\Theta$ and $x$ in $\mathcal{X}$. Then, $F$ is P-Donsker. For any $f$ in $F$, let us define $\mathbb{P}_nf$, $\mathbb{P}f$ and $\mathbb{G}_nf$ as
\begin{align*}
\mathbb{P}_nf = \frac{1}{n}\sum_{i=1}^n f(X_i),\ \mathbb{P}f = \EXP[f(X)],\ \mathbb{G}_nf = \sqrt{n}(\mathbb{P}_nf-\mathbb{P}f).
\end{align*}
Let us also suppose that for all $f$, $\mathbb{P}f^2 < \delta^2$ and $\infnorm{f} < \bM$ and that the random elements $X_1, X_2, \cdots$ are Borel-measurable. Then, we have
\begin{align*}
\EXP \fronorm{\mathbb{G}} = O(1),
\end{align*}
where $\fronorm{\mathbb{G}} = \sup_{f\in F} \abs{ \mathbb{G}_nf }$.
\end{lemma}

\subsection{Proof of Prop.~\ref{prop:opt U}}
\label{supp:sec:opt u}
\begin{proof}
The optimal solution $\bU$ for \eqref{eq:h} is given by the first order optimality condition:
\begin{align*}
\frac{\partial \th(\bY, \bD, \bU)}{\partial \bU} =\ \bU + \lambda_3 (\bU \bY\trans - \bD\trans) \bY = 0,
\end{align*}
by which we have
\begin{align*}
\tilde{\bU} &= \bD\trans \bY \(\frac{1}{\lambda_3}\bI_n + \bY\trans \bY\)^{-1} \\
&=  \bD\trans \(\frac{1}{\lambda_3}\bI_p + \bY\bY\trans\)^{-1} \bY.
\end{align*}
Note that $\tilde{\bu}_i$ is the $i$th column of $\tilde{\bU}$. So for each $i \in [n]$,
\begin{align*}
\tilde{\bu}_i = \bD\trans \(\frac{1}{\lambda_3}\bI_p + \bY\bY\trans\)^{-1} \by_i.
\end{align*}
 Also, we have
\begin{align*}
\bY\tilde{\bU}\trans &= \bY \bY\trans \( \frac{1}{\lambda_3} \bI_p + \bY\bY\trans \)^{-1} \bD\\
&= \( \frac{1}{\lambda_3} \bI_p + \bY\bY\trans - \frac{1}{\lambda_3} \bI_p  \) \( \frac{1}{\lambda_3} \bI_p + \bY\bY\trans \)^{-1} \bD\\
&= \bD - \frac{1}{\lambda_3} \( \frac{1}{\lambda_3} \bI_p + \bY\bY\trans \)^{-1}\bD.
\end{align*}
Thus,
\begin{align*}
h(\bD; \bY) = \frac{1}{2} \fronorm{\bD\trans \(\frac{1}{\lambda_3}\bI_p + \bY\bY\trans\)^{-1} \bY}^2 + \frac{1}{2\lambda_3} \fronorm{  \( \frac{1}{\lambda_3} \bI_p + \bY\bY\trans \)^{-1}\bD }^2.
\end{align*}
\end{proof}

\subsection{Proof of Prop.~\ref{prop:bound:veABMDu}}\label{app:proof:prop:bound}

\begin{proof}
Let us consider the optimization problem of solving $\bv$ and $\be$. As the trivial solution $\{\bv'_t, \be'_t\} = \{\bzero, \bzero\}$ is feasible, we have
\begin{align*}
\tl_1(\bD_{t-1}, \bv'_t, \be'_t; \bz_t) = {\lambda_2} \onenorm{\bz_t}.
\end{align*}
Therefore, the optimal solution should satisfy:
\begin{align*}
\fractwo{\lambda_1} \twonorm{\bz_t - \bD_{t-1}\bv_t^* - \be_t^*}^2 + \fractwo{1}\twonorm{\bv_t^*}^2 + \lambda_2 \onenorm{\be_t^*} \leq {\lambda_2} \onenorm{\bz_t},
\end{align*}
which implies
\begin{align*}
\fractwo{1} \twonorm{\bv_t^*}^2 \leq {\lambda_2} \onenorm{\bz_t},\ \lambda_2 \onenorm{\be_t^*} \leq {\lambda_2} \onenorm{\bz_t}.
\end{align*}
Since $\bz_t$ is uniformly bounded (Assumption~\ref{as:z}) and $\lambda_2$ does not grow with $t$, $\bv_t^*$ and $\be_t^*$ are both uniformly bounded from above.

To examine the uniform boundedness for $\frac{1}{t}\bA_t$ and $\frac{1}{t}\bB_t$, note that
\begin{align*}
&\frac{1}{t}\bA_t = \frac{1}{t} \sum_{i=1}^t \bv_i^* (\bv_i^*)\trans,\\
&\frac{1}{t}\bB_t = \frac{1}{t} \sum_{i=1}^t \(\bz_i - \be_i^*\) (\bv_i^*)\trans.
\end{align*}
Since for each $i$, $\bv_i^*$, $\be_i^*$ and $\bz_i$ are uniformly bounded from above, $\frac{1}{t} \bA_t $ and $\frac{1}{t}\bB_t$ are uniformly upper bounded.

Now we derive the bound for $\bD_t$ and $\bM_t$. We inductively show that both the sequences $\{\bD_t\}_{t \geq 1}$ and $\{ \bM_t \}_{t\geq 1}$ are uniformly bounded. First, let us denote the upper bound of $\frac{1}{t} \bB_t$ by $\const_2$, i.e.,
\begin{align}\label{eq:upper B_t}
\spenorm{\frac{1}{t}\bB_t} \leq \const_2,\ \forall\ t\geq 1.
\end{align}
Also, Assumption~\ref{as:g_t(D)} indicates that there exists an absolute constant $\const_3$, such that
\begin{align}\label{eq:lower A_t}
\spenorm{\frac{1}{t} \bA_t} \geq \const_3,\ \forall\ t\geq 1.
\end{align}
Now suppose that for all $1 \leq i \leq t-1$, it holds for some absolute constant $\const_1 > \const_2/\const_3$ that
\begin{align*}
\spenorm{\bD_i} \leq \const_1,\quad \spenorm{\bM_i} \leq \const_1.
\end{align*}
Using the closed form solution of $\bu_t^*$, we have
\begin{align*}
\bM_t =&\ \bM_{t-1} + \by_t (\bu_t^*)\trans\\
=&\ \bM_{t-1} + \twonorm{\by_t}^2 \(\twonorm{\by_t}^2 + \frac{1}{\lambda_3}\)^{-1} (\bD_{t-1} - \bM_{t-1})\\
=&\ \frac{\twonorm{\by_t}^2}{\twonorm{\by_t}^2 + \frac{1}{\lambda_3}} \bD_{t-1} + \frac{\lambda_3^{-1}}{\twonorm{\by_t}^2 + \frac{1}{\lambda_3}} \bM_{t-1}.
\end{align*}
Hence,
\begin{align*}
\spenorm{\bM_t} \leq \frac{\twonorm{\by_t}^2}{\twonorm{\by_t}^2 + {\lambda_3^{-1}}} \spenorm{\bD_{t-1}} + \frac{\lambda_3^{-1}}{\twonorm{\by_t}^2 + {\lambda_3^{-1}}} \spenorm{\bM_{t-1}} \leq \const_1.
\end{align*}
Now using the closed form solution of $\bD_t$, we have
\begin{align*}
\bD_t = (\lambda_1 \bB_t + \lambda_3 \bM_t) (\lambda_1 \bA_t + \lambda_3 \bI_d)^{-1}= \( \frac{\lambda_1}{t}\bB_t + \frac{\lambda_3}{t} \bM_t\) \( \frac{\lambda_1}{t} \bA_t + \frac{\lambda_3}{t} \bI_d\)^{-1}.
\end{align*}
Combining the above with~\eqref{eq:upper B_t} and~\eqref{eq:lower A_t} gives us
\begin{align*}
\spenorm{\bD_t} \leq \( \lambda_1 \const_2 + \frac{\lambda_3}{t} \const_1 \) \( \lambda_1 \const_3 + \frac{\lambda_3}{t} \)^{-1}= \frac{\lambda_3}{\lambda_1 \const_3} \frac{\const_1 - \const_2 / \const_3}{t + \frac{\lambda_3}{\lambda_1 \const_3}} + \frac{\const_2}{\const_3}.
\end{align*}
It turns out that the maximum of the right hand side is attained at $t = 1$ due to our earlier choice $\const_1 > \const_2 / \const_3$. Hence, we have
\begin{align*}
\spenorm{\bD_t} \leq \const_1.
\end{align*}
The induction is complete.

By examining the closed form of $\bu_t^*$, and note that we have proved the uniform boundedness of $\bD_t$ and $\bM_t$, we conclude that $\bu_t^*$ is uniformly upper bounded.
\end{proof}

\begin{corollary}
\label{coro:bound l lip gt}
Assume same conditions as in Prop.~\ref{prop:bound:veABMDu}. Further suppose that $\lambda_1$ does not grow with $t$. Then, for all $t \geq 1$,

\begin{enumerate}
\item  $\tl(\bD_t, \bv_t^*, \be_t^*; \bz_t)$ and $\ell(\bD_t; \bz_t)$ are uniformly bounded from above.

\item $\frac{1}{t} \th(\bD_t, \bU^*_{1:t}; \bY_{1:t})$ is uniformly upper bounded where $\bU^*_{1:t} = (\bu^*_1, \dots, \bu^*_t)$ and $\bY_{1:t} = (\by_1, \dots, \by_t)$.

\item The surrogate function $g_t(\bD_t)$ defined in \eqref{eq:g_t(D)} is uniformly upper bounded and Lipschitz.
\end{enumerate}
\end{corollary}
\begin{proof}
To show Claim 1, we just need to examine the definition of $\tl(\bD_t, \bv_t^*, \be_t^*; \bz_t)$ (see Eq.~\eqref{eq:ell}) and notice that $\bz_t$, $\bD_t$, $\bv_t^*$ and $\be_t^*$ are all uniformly bounded. This implies that $\tl(\bD_t, \bv_t^*, \be_t^*; \bz_t)$ is uniformly bounded and so is $\ell(\bD_t; \bz_t)$. Likewise, we prove that $\frac{1}{t} \th(\bD_t, \bU^*_{1:t}; \bY_{1:t})$ is uniformly bounded. The uniform boundedness of $g_t(\bD_t)$ follows immediately.

To show that $g_t(\bD)$ is Lipschitz, we show that the gradient of $g_t(\bD)$ is uniformly bounded for all $\bD \in \mathcal{D}$.
\begin{align*}
\fronorm{\nabla g_t(\bD)} =&\ \fronorm{ \lambda_1 \bD \(\frac{1}{t}\bA_t + \frac{\lambda_3}{t} \bI_d\) -  \frac{\lambda_1}{t}\bB_t - \frac{\lambda_3}{t} \bM_t }\\
\leq &\ \lambda_1 \fronorm{\bD} \(\fronorm{\frac{1}{t} \bA_t} + \fronorm{\frac{\lambda_3}{t} \bI_d}\) + \lambda_1 \fronorm{ \frac{1}{t} \bB_t}  + \fronorm{\frac{\lambda_3}{t} \bM_t}.
\end{align*}
Notice that each term on the right hand side of the inequality is uniformly bounded from above and $\lambda_1$ does not grow with $t$. Thus the gradient of $g_t(\bD)$ is uniformly bounded, implying that $g_t(\bD)$ is Lipschitz.
\end{proof}

\begin{proposition}
\label{prop:l:Lipschtiz}
Let $\bD \in \mathcal{D}$ and denote the minimizer of $\tl(\bD, \bv, \be; \bz)$ as:
\begin{align*}
\{\bv', \be'\} = \argmin_{\bv, \be} \tl(\bD, \bv, \be; \bz).
\end{align*}
 Then, the function $\ell(\bD; \bz)$ is continuously differentiable and
\begin{align*}
\nabla_{\bD} \ell(\bD; \bz) = (\bD \bv' + \be' - \bz) (\bv')\trans.
\end{align*}
Furthermore, $\ell(\bD; \bz)$ is uniformly Lipschitz.
\end{proposition}
\begin{proof}
By fixing $\bz$, the function $\tilde{\ell}$ can be seen as a mapping:
\begin{align*}
\mathbb{R}^{d+p} \times \mathcal{D} \rightarrow&\ \mathbb{R}\\
\( [\bv;\ \be], \bD \) \mapsto&\ \tilde{\ell}(\bD, \bv, \be; \bz).
\end{align*}
It is easy to show that for all $[\bv;\ \be] \in \mathbb{R}^{d+p}$, $\tl(\bD, \bv, \be; \bz)$ is differentiable with respect to $\bD$. Also $\tl(\bD, \bv, \be; \bz)$ is continuous on $\mathbb{R}^{d+p} \times \mathcal{D}$ and so is its gradient $\nabla^{}_{\bD} \tl(\bD, \bv, \be; \bz) = (\bD \bv + \be - \bz)\bv\trans$. For all $\bD \in \mathcal{D}$, since $\tl(\bD, \bv, \be; \bz)$ is strongly convex w.r.t. $\bv$ and $\be$, it has a unique minimizer $\{\bv', \be'\}$. Thus Lemma~\ref{lem:bonnans} applies and we prove that $\ell(\bD; \bz)$ is differentiable in $\bD$ and
\begin{align*}
\nabla_{\bD} \ell(\bD; \bz) = (\bD \bv' + \be' - \bz) (\bv')\trans.
\end{align*}
Since every term in $\nabla_{\bD} \ell(\bD; \bz)$ is uniformly bounded (Assumption~\ref{as:z} and Proposition~\ref{prop:bound:veABMDu}), we conclude that the gradient of $\ell(\bD; \bz)$ is uniformly bounded, implying that $\ell(\bz, \bD)$ is uniformly Lipschitz w.r.t. $\bD$.
\end{proof}

\begin{corollary}
\label{coro:bound lip ft}
Let $f_t(\bD)$ be the empirical loss function defined in \eqref{eq:f_n(D)}. Then $f_t(\bD)$ is uniformly bounded from above and Lipschitz for all $t \geq 1$ and $\bD \in \mathcal{D}$.
\end{corollary}
\begin{proof}
Since $\ell(\bD; \bz)$ can be uniformly bounded (Corollary~\ref{coro:bound l lip gt}), we only need to show that $\frac{1}{t} h(\bD; \bY)$ is uniformly bounded, where $\bY = (\by_1, \by_2, \dots, \by_t)$. Note that we have derived the form for $h(\bD; \bY)$ in Prop.~\ref{prop:opt U}:
\begin{align*}
\frac{1}{t}h(\bD; \bY) = \frac{1}{2t} \sum_{i=1}^{t} \twonorm{ \bD\trans \(\frac{1}{\lambda_3}\bI_p +  \bY\bY\trans\)^{-1} \by_i }^2 + \frac{1}{2\lambda_3 t} \fronorm{  \( \frac{1}{\lambda_3} \bI_p + \bY\bY\trans \)^{-1}\bD }^2.
\end{align*}
Since every term in the above equation can be uniformly bounded, $\frac{1}{t}h(\bD; \bZ_t)$ is uniformly bounded and so is $f_t(\bD)$.

To show that $f_t(\bD)$ is uniformly Lipschitz, it amounts to prove that its gradient can be uniformly bounded from above. Using Prop.~\ref{prop:l:Lipschtiz}, we have
\begin{align*}
\nabla f_t(\bD) =&\ \frac{1}{t} \sum_{i=1}^{t} \nabla \ell(\bD; \bz_i) + \frac{1}{t} \nabla h(\bD; \bZ_t) \\
=&\ \frac{1}{t} \sum_{i=1}^{t} (\bD \bv_i^* + \be_i^* - \bz_i) (\bv_i^*)\trans \\
&\ + \frac{1}{t} \sum_{i=1}^{t} \(\frac{1}{\lambda_3}\bI_p + \bY\bY\trans\)^{-1} \by_i \by_i\trans \(\frac{1}{\lambda_3}\bI_p +  \bY\bY\trans\)^{-1} \bD \\
&\ + \frac{\lambda_3}{t} \(\frac{1}{\lambda_3}\bI_p +  \bY\bY\trans\)^{-2} \bD.
\end{align*}
Then the Frobenius norm of $\nabla f_t(\bD)$ can be bounded by:
\begin{align*}
\fronorm{\nabla f_t(\bD)}\leq&\ \frac{1}{t} \sum_{i=1}^{t} \twonorm{ \bD \bv_i^* + \be_i^* - \bz_i } \cdot \twonorm{ \bv_i^* } \\
&\ + \frac{1}{t} \sum_{i=1}^{t} \fronorm{ \(\frac{1}{\lambda_3}\bI_p +  \bY\bY\trans\)^{-1} }^2 \cdot \twonorm{\by_i}^2 \cdot \fronorm{\bD}\\
&\ + \frac{\lambda_3}{t} \fronorm{ \(\frac{1}{\lambda_3}\bI_p +  \bY\bY\trans\)^{-1} }^2 \cdot \fronorm{\bD}.
\end{align*}
One can easily check that the right hand side of the inequality is uniformly bounded from above. Thus $\fronorm{\nabla f_t(\bD)}$ is uniformly bounded, implying that $f_t(\bD)$ is uniformly Lipschitz.
\end{proof}

\subsection{Proof of P-Donsker}
\begin{proposition}
\label{prop:l donsker}
Let $\tilde{f}_t(\bD) = \frac{1}{t} \sum_{i=1}^{t} \ell(\bD; \bz_i)$. Then we have
\begin{align*}
\EXP[\sqrt{t} \infnorm{\tilde{f}_t -  \EXP_{\bz}[\ell(\bD; \bz)] }] = O(1).
\end{align*}
\end{proposition}
\begin{proof}
Let us consider $\{ \ell(\bD; \bz) \}$ as a set of measurable functions indexed by $\bD \in \mathcal{D}$. As we showed in Proposition~\ref{prop:bound:veABMDu}, $\mathcal{D}$ is a compact set. Also, we have proved that $\ell(\bD; \bz)$ is uniformly Lipschitz over $\bD$ (Proposition~\ref{prop:l:Lipschtiz}). Thus, $\{ \ell(\bD; \bz) \}$ is P-Donsker (see the definition in Lemma~\ref{lem:donsker}). Furthermore, as $\ell(\bD; \bz)$ is non-negative and its magnitude is uniformly upper bounded~(Corollary~\ref{coro:bound l lip gt}), so is $\ell^2(\bD; \bz)$. Hence we have $\EXP_{\bz}[\ell^2(\bD; \bz)] \leq c$ for some absolute constant $c$. Note that we have verified all the hypotheses in Lemma~\ref{lem:donsker}. Hence the proof is complete.
\end{proof}

\subsection{Proof of Theorem~\ref{thm:convergence g_t(D_t)}}

\begin{proof}
Note that $g_t(\bD_t)$ can be viewed as a stochastic positive process since every term in $g_t(\bD_t)$ is non-negative and the samples are drawn randomly. We define for all $t \geq 1$
\begin{align*}
\psi_t \defeq g_t(\bD_t).
\end{align*}
To show the convergence of $\psi_t$, we need to bound the difference of $\psi_{t+1}$ and $\psi_t$:
\begin{align}
\label{eq:pf:diff u_t}
 \psi_{t+1} - \psi_t =&\ g_{t+1}(\bD_{t+1}) - g_t(\bD_t) \notag\\
=&\ g_{t+1}(\bD_{t+1}) - g_{t+1}(\bD_t) + g_{t+1}(\bD_t) - g_t(\bD_t)\notag\\
=&\ g_{t+1}(\bD_{t+1}) - g_{t+1}(\bD_t)  + \frac{1}{t+1} \ell(\bz_{t+1}, \bD_t) - \frac{1}{t+1} \tilde{g}_t(\bD_t) \notag\\
&\ + \Bigg[ \frac{1}{t+1} \sum_{i=1}^{t+1} \fractwo{1} \twonorm{ \bu_i^* }^2 + \frac{\lambda_3}{2(t+1)} \fronorm{\bD_t - \bM_{t+1}}^2 \notag\\
&\ - \frac{1}{t} \sum_{i=1}^{t} \frac{1}{2} \twonorm{ \bu_i^* }^2 - \frac{\lambda_3}{2t} \fronorm{\bD_t - \bM_{t}}^2 \Bigg]\notag\\
=&\ g_{t+1}(\bD_{t+1}) - g_{t+1}(\bD_t) + \frac{\tilde{f}_t(\bD_t) - \tilde{g}_t(\bD_t)}{t+1} + \frac{\ell(\bz_{t+1}, \bD_t) - \tilde{f}_t(\bD_t)}{t+1} \notag\\
&\ + \Bigg[ \frac{1}{t+1} \sum_{i=1}^{t+1} \fractwo{1} \twonorm{ \bu_i^* }^2 + \frac{\lambda_3}{2(t+1)} \fronorm{\bD_t - \bM_{t+1}}^2 \notag\\
&\ - \frac{1}{t} \sum_{i=1}^{t} \frac{1}{2} \twonorm{ \bu_i^* }^2 - \frac{\lambda_3}{2t} \fronorm{\bD_t - \bM_{t}}^2 \Bigg].
\end{align}
Here,
\begin{align*}
\tilde{g}_t(\bD_t) = \frac{1}{t} \sum_{i=1}^{t} \tl(\bD, \bv_i, \be_i; \bz_i),
\end{align*}
and
\begin{align*}
\tilde{f}_t(\bD_t) = \frac{1}{t} \sum_{i=1}^{t} \ell(\bD_t; \bz_i).
\end{align*}
First, we bound the four terms in the brackets of~\eqref{eq:pf:diff u_t}. We have
\begin{align}
\label{eq:pf:1}
 \frac{1}{t+1} \sum_{i=1}^{t+1} \fractwo{1} \twonorm{ \bu_i^* }^2 - \frac{1}{t} \sum_{i=1}^{t} \twonorm{ \bu_i^* }^2 =&\ \frac{-1}{t(t+1)} \sum_{i=1}^{t} \fractwo{1} \twonorm{ \bu_i^* }^2 + \frac{1}{2(t+1)} \twonorm{\bu_{t+1}^*}^2\notag\\
\leq&\ \frac{1}{2(t+1)} \twonorm{\bu_{t+1}^*}^2,
\end{align}
and
\begin{align}
\label{eq:pf:2}
&\ \frac{\lambda_3}{2(t+1)} \fronorm{\bD_t - \bM_{t+1}}^2 - \frac{\lambda_3}{2t} \fronorm{\bD_t - \bM_{t}}^2 \notag\\
=&\ \frac{-\lambda_3}{2t(t+1)} \fronorm{\bD_t - \bM_t}^2 + \frac{\lambda_3}{2(t+1)} \fronorm{ \bz_{t+1} (\bu_{t+1}^*)\trans }^2 \notag\\
&\ - \frac{\lambda_3}{t+1} \tr\( (\bD_t - \bM_t)\trans \bz_{t+1} (\bu_{t+1}^*)\trans \)\notag\\
=&\ \frac{-\lambda_3}{2t(t+1)} \fronorm{\bD_t - \bM_t}^2 + \frac{\lambda_3}{2(t+1)} \fronorm{ \bz_{t+1} (\bu_{t+1}^*)\trans }^2 \notag\\
&\ - \frac{\lambda_3}{t+1} \( \twonorm{ \bz_{t+1} }^2 + \frac{1}{\lambda_3} \) \twonorm{ \bu_{t+1}^* }^2\notag\\
\leq&\ \frac{1}{t+1} \( \frac{\lambda_3}{2} \fronorm{ \bz_{t+1} (\bu_{t+1}^*)\trans }^2 -  (\lambda_3 \twonorm{ \bz_{t+1} }^2 + 1 ) \twonorm{ \bu_{t+1}^* }^2 \)\notag\\
\leq&\ \frac{1}{t+1} \( -\fractwo{\lambda_3} \twonorm{ \bz_{t+1} }^2 \twonorm{ \bu_{t+1}^* }^2 - \twonorm{ \bu_{t+1}^* }^2 \),
\end{align}
where the first equality is derived by the fact that $\bM_{t+1} = \bM_t + \bz_{t+1} (\bu_{t+1}^*)\trans$, and the second equality is derived by the closed form solution of $\bu_{t+1}^*$ (see \eqref{eq:u_t}).

Combining \eqref{eq:pf:1} and \eqref{eq:pf:2}, we know that
\begin{align*}
&\ \frac{1}{t+1} \sum_{i=1}^{t+1} \fractwo{1} \twonorm{ \bu_i^* }^2 - \frac{1}{t} \sum_{i=1}^{t} \twonorm{ \bu_i^* }^2 \\
&\ + \frac{\lambda_3}{2(t+1)} \fronorm{\bD_t - \bM_{t+1}}^2 - \frac{\lambda_3}{2t} \fronorm{\bD_t - \bM_{t}}^2\\
\leq&\ \frac{1}{2(t+1)} \twonorm{\bu_{t+1}^*}^2 + \frac{1}{t+1} \Big( -\fractwo{\lambda_3} \twonorm{ \bz_{t+1} }^2 \twonorm{ \bu_{t+1}^* }^2  - \twonorm{ \bu_{t+1}^* }^2 \Big)\\
=&\ \frac{1}{t+1} \( -\fractwo{\lambda_3} \twonorm{ \bz_{t+1} }^2 \twonorm{ \bu_{t+1}^* }^2 - \fractwo{1} \twonorm{ \bu_{t+1}^* }^2 \) \leq 0.
\end{align*}

Therefore,
\begin{align*}
 \psi_{t+1} - \psi_t \leq&\ g_{t+1}(\bD_{t+1}) - g_{t+1}(\bD_t) + \frac{1}{t+1} \ell(\bz_{t+1}, \bD_t)  - \frac{1}{t+1} \tilde{g}_t(\bD_t)\\
=&\ g_{t+1}(\bD_{t+1}) - g_{t+1}(\bD_t) + \frac{\tilde{f}_t(\bD_t) - \tilde{g}_t(\bD_t)}{t+1}  + \frac{\ell(\bz_{t+1}, \bD_t) - \tilde{f}_t(\bD_t)}{t+1}\\
\leq&\ \frac{\ell(\bz_{t+1}, \bD_t) - \tilde{f}_t(\bD_t)}{t+1},
\end{align*}
where the last inequality holds because $\bD_{t+1}$ is the minimizer of $g_{t+1}(\bD)$ and $\tilde{g}_t(\bD)$ is a surrogate function of $\tilde{f}_t(\bD)$.

Let $\mathcal{F}_t$ be the filtration of the past information. We take the expectation on the above equation conditional on $\mathcal{F}_t$:
\begin{align*}
\EXP[\psi_{t+1} - \psi_t \mid \mathcal{F}_t] \leq&\ \frac{ \EXP[\ell(\bz_{t+1}, \bD_t) \mid \mathcal{F}_t] - \tilde{f}_t(\bD_t) }{t+1} \\
\leq&\ \frac{ f(\bD_t) - \tilde{f}_t(\bD_t) }{t+1}\\
\leq&\ \frac{\infnorm{f - \tilde{f}_t}}{t+1}.
\end{align*}
From Proposition~\ref{prop:l donsker}, we know
\begin{align*}
\EXP\Big[\infnorm{f- \tilde{f}_t}\Big] = O\(\frac{1}{\sqrt{t}}\).
\end{align*}
Thus,
\begin{align}\label{eq:tmp:diff psi}
\EXP\big[\EXP[\psi_{t+1} - \psi_t \mid \mathcal{F}_t]^+\big] = \EXP\big[\max\{ \EXP[\psi_{t+1} - \psi_t \mid \mathcal{F}_t], 0 \}\big] \leq \frac{c}{\sqrt{t}(t+1)},
\end{align}
where $c$ is some constant.

Now let us define the index set
\begin{align*}
\mathcal{T} = \Big\{ t:\ t \geq 1, \EXP\big[\psi_{t+1} - \psi_t \mid \mathcal{F}_t \big] > 0 \Big\},
\end{align*}
and the indicator function
\begin{align*}
\delta_t =
\begin{cases}
1,\quad \textrm{if}\ t \in \mathcal{T},\\
0,\quad \textrm{otherwise}.
\end{cases}
\end{align*}
It follows that
\begin{align*}
\sum_{t=1}^{\infty} \EXP[\delta_t(\psi_{t+1} - \psi_t)] =&\ \sum_{t \in \mathcal{T}} \EXP[\psi_{t+1} - \psi_t ]\\
=&\ \sum_{t \in \mathcal{T}} \EXP[\EXP[\psi_{t+1} - \psi_t \mid \mathcal{F}_t]]\\
=&\ \sum_{t=1}^{\infty} \EXP[\EXP[\psi_{t+1} - \psi_t \mid \mathcal{F}_t]^+]\\
\leq& +\infty,
\end{align*}
where the last inequality holds in view of~\eqref{eq:tmp:diff psi}.

Thus, Lemma~\ref{lem:bottou} applies. That is, $\{g_t(\bD_t)\}_{t\geq 1}$ is a quasi-martingale and converges almost surely. In addition,
\begin{equation*}
\sum_{t=1}^{\infty} \abs{ \EXP\big[\psi_{t+1} - \psi_t \mid \mathcal{F}_t\big] } < +\infty, \ a.s.
\end{equation*}
\end{proof}

\subsection{Proof of Prop.~\ref{prop:diff_D}}\label{app:proof:prop:diff_D}

\begin{proof}
According the strong convexity of $g_t(\bD)$ (Assumption~\ref{as:g_t(D)}), we have,
\begin{equation}
\label{eq:pf:dif_g1}
g_t(\bD_{t+1}) - g_t(\bD_t) \geq \frac{\beta_0}{2} \fronorm{\bD_{t+1} - \bD_t}^2,
\end{equation}
On the other hand,
\begin{align}
\label{eq:pf:dif_g2}
&\ g_t(\bD_{t+1}) - g_t(\bD_t) \notag\\
=&\ g_t(\bD_{t+1}) - g_{t+1}(\bD_{t+1}) + g_{t+1}(\bD_{t+1}) - g_{t+1}(\bD_t) \notag\\
&\ + g_{t+1}(\bD_t) - g_t(\bD_t)\notag\\
\leq&\ g_t(\bD_{t+1}) - g_{t+1}(\bD_{t+1}) + g_{t+1}(\bD_t) - g_t(\bD_t).
\end{align}
Note that the inequality is derived by the fact that $g_{t+1}(\bD_{t+1}) - g_{t+1}(\bD_t) \leq 0$, as $\bD_{t+1}$ is the minimizer of $g_{t+1}(\bD)$. Let
\begin{align}
G_t(\bD) = g_t(\bD) - g_{t+1}(\bD).
\end{align}
By a simple calculation, we obtain the gradient of $G_t(\bD)$:
\begin{align*}
 \nabla G_t(\bD) =&\ \nabla g_t(\bD) - \nabla g_{t+1}(\bD)\\
=&\ \frac{1}{t} \Big[ \bD\(\lambda_1 \bA_t + \lambda_3 \bI_d\) - \(\lambda_1 \bB_t + \lambda_3 \bM_t\) \Big] \\
&\ - \frac{1}{t+1} \Big[ \bD(\lambda_1 A_{t+1} + \lambda_3 \bI_d) - (\lambda_1 B_{t+1} + \lambda_3 \bM_{t+1}) \Big]\\
=&\ \frac{1}{t} \Bigg[ \bD\(\lambda_1 \bA_t + \lambda_3 \bI_d - \frac{\lambda_1 t}{t+1} A_{t+1} - \frac{\lambda_3 t}{t+1} \bI_d\) \\
&\ + \frac{\lambda_1 t}{t+1} B_{t+1}  - \lambda_1 \bB_t + \frac{\lambda_3 t}{t+1} \bM_{t+1} - \lambda_3 \bM_t \Bigg]\\
=&\ \frac{1}{t} \Bigg[ \bD\(\frac{\lambda_1}{t+1} A_{t+1} - \lambda_1 \bv_{t+1} \bv_{t+1}\trans + \frac{\lambda_3}{t+1} \bI_d\) \\
&\ + \lambda_1(\bz_{t+1}  - \be_{t+1}) \bv_{t+1}\trans - \frac{\lambda_1}{t+1}B_{t+1}  + \lambda_3 \bz_{t+1} \bu_{t+1}\trans - \frac{\lambda_3}{t+1} \bM_{t+1} \Bigg]
\end{align*}
So the upper bound of the Frobenius norm of $\nabla G_t(\bD)$ follows immediately:
\begin{align*}
&\ \fronorm{\nabla G_t(\bD)}\\
 \leq &\ \frac{1}{t} \Bigg[ \fronorm{\bD}\Bigg(\lambda_1 \fronorm{\frac{A_{t+1}}{t+1}} + \lambda_1 \fronorm{\bv_{t+1} \bv_{t+1}\trans}  + \frac{\lambda_3}{t+1} \fronorm{\bI_d}\Bigg)  \\
&\ + \lambda_1 \fronorm{(\bz_{t+1} - \be_{t+1})\bv_{t+1}\trans}  + \lambda_1 \fronorm{\frac{B_{t+1}}{t+1}} + \lambda_3 \fronorm{\bz_{t+1}\bu_{t+1}\trans} + \frac{\lambda_3}{t+1} \fronorm{\bM_{t+1}} \Bigg]\\
=&\ \frac{1}{t} \Bigg[ \fronorm{\bD}\(\lambda_1 \fronorm{\frac{A_{t+1}}{t+1}} + \lambda_1 \fronorm{\bv_{t+1} \bv_{t+1}\trans}\)  + \lambda_1 \fronorm{(\bz_{t+1} - \be_{t+1})\bv_{t+1}\trans} \\
&\  + \lambda_1 \fronorm{\frac{B_{t+1}}{t+1}} + \lambda_3 \fronorm{\bz_{t+1}\bu_{t+1}\trans}  \Bigg]  + \frac{\lambda_3}{t(t+1)} \Big[ \fronorm{\bI_d} + \fronorm{\bM_{t+1}} \Big].
\end{align*}
We know from Proposition~\ref{prop:bound:veABMDu} that all the terms in the above equation are uniformly bounded from above. Thus, there exist constants $c_1$, $c_2$ and $c_3$, such that
\begin{align}\label{eq:grad G_t}
\fronorm{\nabla G_t(\bD)} \leq \frac{1}{t} \( c_1 \fronorm{\bD} + c_2 \) + \frac{c_3}{t}.
\end{align}

According to the first order Taylor expansion,
\begin{align*}
&\ G_t(\bD_{t+1}) - G_t(\bD_t) \\
= &\ \tr\( \(\bD_{t+1} - \bD_t\)\trans \nabla G_t \(\rho \bD_t + \(1- \rho\)\bD_{t+1}\) \)\\
\leq &\ \fronorm{\bD_{t+1} - \bD_t} \cdot \fronorm{\nabla G_t \(\rho\bD_t + \(1- \rho\)\bD_{t+1}\)},
\end{align*}
where $\rho$ is some scalar between 0 and 1. According to Proposition~\ref{prop:bound:veABMDu}, $\bD_t$ and $\bD_{t+1}$ are uniformly bounded, indicating that $\rho \bD_t + \(1- \rho\)\bD_{t+1}$ is uniformly bounded. In view of~\eqref{eq:grad G_t}, there exists a constant $c_4$, such that
\begin{align*}
\fronorm{\nabla G_t \(\alpha \bD_t + \(1- \alpha\) \bD_{t+1}\)} \leq \frac{c_4}{t},
\end{align*}
resulting in
\begin{align*}
G_t(\bD_{t+1}) - G_t(\bD_t) \leq \frac{c_4}{t} \cdot \fronorm{\bD_{t+1} - \bD_t}.
\end{align*}
Combining \eqref{eq:pf:dif_g1}, \eqref{eq:pf:dif_g2} and the above equation, we have
\begin{align*}
\fronorm{\bD_{t+1} - \bD_t} = \frac{2c_4}{\beta_0 t}.
\end{align*}
\end{proof}

\subsection{Proof of Theorem~\ref{thm:convergence f_t(D_t)}}

\begin{proof}
Recall that $\tilde{f}_t(\bD) = \frac{1}{t} \sum_{i=1}^{t} \ell(\bD; \bz_i)$ and $\tilde{g}_t(\bD) = \frac{1}{t} \sum_{i=1}^{t} \tl(\bD, \bv_i^*, \be_i^*; \bz_i)$. Define
\begin{align*}
b_t \defeq&\ g_t(\bD_t) - f_t(\bD_t)\\
=&\ \tilde{g}_t(\bD_t) - \tilde{f}_t(\bD_t) + \Bigg[ \frac{1}{t} \sum_{i=1}^{t} \fractwo{1} \twonorm{\bu_i^*}^2 + \frac{\lambda_3}{2t} \fronorm{\bD_t - \bM_t}^2 \\
&\ - \frac{1}{t} \sum_{i=1}^t \fractwo{1} \twonorm{  \bD_t\trans \(\frac{1}{\lambda_3}\bI_p +  \bY\bY\trans\)^{-1} \bz_i }^2  - \frac{1}{2 \lambda_3} \fronorm{ \( \frac{1}{\lambda_3} \bI_p + \bY\bY\trans \)^{-1}\bD_t }^2 \Bigg]\\
=&\ \tilde{g}_t(\bD_t) - \tilde{f}_t(\bD_t) + q_t(\bD_t),
\end{align*}
where $q_t(\bD_t)$ denotes the four terms in the brackets. Using~\eqref{eq:pf:diff u_t}, we have
\begin{align*}
\frac{b_t}{t+1} =&\ \frac{\tilde{g}_t(\bD_t) - \tilde{f}_t(\bD_t)}{t+1} + \frac{q_t(\bD_t)}{t+1}\\
=&\ g_{t+1}(\bD_{t+1}) - g_{t+1}(\bD_t) + \frac{ \ell( \bD_t; \bz_{t+1}) - \tilde{f}_t(\bD_t) }{t+1}  + \psi_t - \psi_{t+1} \\
&\ + \Bigg[ \frac{q_t(\bD_t)}{t+1} +  \frac{1}{t+1} \sum_{i=1}^{t+1} \fractwo{1} \twonorm{ \bu_i^* }^2  + \frac{\lambda_3}{2(t+1)} \fronorm{\bD_t - \bM_{t+1}}^2 \\
&\ - \frac{1}{t} \sum_{i=1}^{t} \frac{1}{2} \twonorm{ \bu_i^* }^2  - \frac{\lambda_3}{2t} \fronorm{\bD_t - \bM_{t}}^2 \Bigg].
\end{align*}
Note that it always holds that for some constant $c$,
\begin{align*}
\frac{q_t(\bD_t)}{t+1} \leq&\ \frac{1}{t(t+1)} \sum_{i=1}^{t} \fractwo{1} \twonorm{\bu_i^*}^2 + \frac{\lambda_3}{2t(t+1)} \fronorm{\bD_t - \bM_t}^2\\
\leq&\ \frac{1}{t(t+1)} \sum_{i=1}^{t} \fractwo{1} \twonorm{\bu_i^*}^2 + \frac{c}{2t(t+1)},
\end{align*}
where the second inequality is due to the fact that $\bD_t$ and $\bM_t$ are both uniformly bounded (see Proposition~\ref{prop:bound:veABMDu}).

On the other hand, from~\eqref{eq:pf:1} we know that
\begin{align*}
\frac{1}{t+1} \sum_{i=1}^{t+1} \fractwo{1} \twonorm{ \bu_i^* }^2 - \frac{1}{t} \sum_{i=1}^{t} \twonorm{ \bu_i^* }^2 = \frac{-1}{t(t+1)} \sum_{i=1}^{t} \fractwo{1} \twonorm{ \bu_i^* }^2  + \frac{1}{2(t+1)} \twonorm{\bu_{t+1}^*}^2,
\end{align*}
while~\eqref{eq:pf:2} implies
\begin{align*}
 \frac{\lambda_3}{2(t+1)} \fronorm{\bD_t - \bM_{t+1}}^2 - \frac{\lambda_3}{2t} \fronorm{\bD_t - \bM_{t}}^2\leq \frac{1}{t+1} \( -\fractwo{\lambda_3} \twonorm{ \bz_{t+1} }^2 \twonorm{ \bu_{t+1}^* }^2 - \twonorm{ \bu_{t+1}^* }^2 \).
\end{align*}
Combining these pieces, we have
\begin{align*}
&\ \frac{q_t(\bD_t)}{t+1} +  \frac{1}{t+1} \sum_{i=1}^{t+1} \fractwo{1} \twonorm{ \bu_i^* }^2  + \frac{\lambda_3}{2(t+1)} \fronorm{\bD_t - \bM_{t+1}}^2 \\
&\ - \frac{1}{t} \sum_{i=1}^{t} \frac{1}{2} \twonorm{ \bu_i^* }^2 - \frac{\lambda_3}{2t} \fronorm{\bD_t - \bM_{t}}^2 \\
\leq&\ \frac{ c}{2t(t+1)} + \frac{1}{2(t+1)} \twonorm{\bu_{t+1}^*}^2  + \frac{1}{t+1} \( -\fractwo{\lambda_3} \twonorm{ \bz_{t+1} }^2 \twonorm{ \bu_{t+1}^* }^2 - \twonorm{ \bu_{t+1}^* }^2 \)\\
=&\ \frac{ c}{2t(t+1)} - \frac{1}{2(t+1)} \twonorm{\bu_{t+1}^*}^2  - \frac{\lambda_3}{2(t+1)} \twonorm{ \bz_{t+1} }^2 \twonorm{ \bu_{t+1}^* }^2\\
\leq&\ \frac{ c}{2t(t+1)}.
\end{align*}
Therefore,
\begin{align*}
 \frac{b_t}{t+1} \leq&\ g_{t+1}(\bD_{t+1}) - g_{t+1}(\bD_t) + \frac{ \ell(\bD_t; \bz_{t+1}) - \tilde{f}_t(\bD_t) }{t+1} \\
&\ + \psi_t - \psi_{t+1} + \frac{ c}{2t(t+1)}\\
\leq&\ \frac{ \ell(\bD_t; \bz_{t+1}) - \tilde{f}_t(\bD_t) }{t+1} + \psi_t - \psi_{t+1} + \frac{ c}{2t(t+1)},
\end{align*}
where we use the fact that $\bD_{t+1}$ minimizes $g_{t+1}(\bD)$ in the second inequality and we denote $\psi_t = g_t(\bD_t)$. By taking the expectation over $\bz$ conditional on the past filtration $\mathcal{F}_t$, we have
\begin{align*}
\frac{b_t}{t+1} \leq& \frac{c_1}{\sqrt{t}(t+1)} + \abs{ \EXP[\psi_t - \psi_{t+1} \mid \mathcal{F}_t] } + \frac{ c}{2t(t+1)},
\end{align*}
which is an immediate result from Prop.~\ref{prop:l donsker}. Thereby,
\begin{align*}
 \sum_{t=1}^{\infty} \frac{b_t}{t+1} \leq \sum_{t=1}^{\infty} \frac{c_1}{\sqrt{t}(t+1)} + \sum_{t=1}^{\infty} \abs{ \EXP[\psi_t - \psi_{t+1} \mid \mathcal{F}_t] } + \sum_{t=1}^{\infty} \frac{ c}{2t(t+1)} < +\infty.
\end{align*}
Here, the last inequality is derived by applying \eqref{eq:pf:sum diff u_t}.

Next, we examine the difference between $b_{t+1}$ and $b_t$:
\begin{align}\label{eq:pf:diff b_t}
 \abs{ b_{t+1} - b_t } =&\ \abs{ g_{t+1}(\bD_{t+1}) -  f_{t+1}(\bD_{t+1}) - g_t(\bD_t) + f_t(\bD_t) } \notag\\
\leq&\  \abs{ g_{t+1}(\bD_{t+1}) - g_t(\bD_{t+1}) } + \abs{ g_t(\bD_{t+1}) - g_t(\bD_t) } \notag\\
&\ + \abs{ f_{t+1}(\bD_{t+1}) - f_t(\bD_{t+1}) } + \abs{ f_t(\bD_{t+1}) - f_t(\bD_t) }.
\end{align}
For the first term on the right hand side, we have
\begin{align*}
&\ \abs{ g_{t+1}(\bD_{t+1}) - g_t(\bD_{t+1}) }\\
=&\ \Big| \tilde{g}_{t+1}(\bD_{t+1}) - \tilde{g}_t(\bD_{t+1}) + \frac{1}{t+1} \sum_{i=1}^{t+1} \fractwo{1} \twonorm{\bu_i^*}^2  - \frac{1}{t} \sum_{i=1}^{t} \fractwo{1} \twonorm{\bu_i^*}^2 \\
&\ + \frac{\lambda_3}{2(t+1)} \fronorm{\bD_{t+1} - \bM_{t+1}}^2  - \frac{\lambda_3}{2t} \fronorm{\bD_{t+1} - \bM_t}^2 \Big|\\
=&\ \Big| \tilde{g}_{t+1}(\bD_{t+1}) - \tilde{g}_t(\bD_{t+1}) - \frac{1}{t(t+1)} \sum_{i=1}^{t} \fractwo{1} \twonorm{\bu_i^*}^2  - \frac{1}{2(t+1)} \twonorm{\bu_{t+1}^*}^2 \\
&\ - \frac{\lambda_3}{2t(t+1)} \fronorm{\bD_{t+1} - \bM_t}^2 - \frac{\lambda_3}{2(t+1)} \fronorm{\bz_{t+1}(\bu_{t+1}^*)\trans}^2 \Big| \\
\leq&\ \abs{ \tilde{g}_{t+1}(\bD_{t+1}) - \tilde{g}_t(\bD_{t+1}) } +  \frac{1}{t(t+1)} \sum_{i=1}^{t} \fractwo{1} \twonorm{\bu_i^*}^2   + \frac{1}{2(t+1)} \twonorm{\bu_{t+1}^*}^2 \\
&\ + \frac{\lambda_3}{2t(t+1)} \fronorm{\bD_{t+1} - \bM_t}^2  + \frac{\lambda_3}{2(t+1)} \fronorm{\bz_{t+1}(\bu_{t+1}^*)\trans}^2\\
\stackrel{\zeta_1}{\leq}&\ \abs{ \tilde{g}_{t+1}(\bD_{t+1}) - \tilde{g}_t(\bD_{t+1}) } + \frac{c_1}{t+1}\\
=&\ \abs{ \frac{1}{t+1} \ell(\bD_{t+1}; \bz_{t+1}) - \frac{1}{t+1} \tilde{g}_t(\bD_{t+1}) } + \frac{c_1}{t+1}\\
\stackrel{\zeta_2}{\leq}&\ \frac{c_2}{t+1},
\end{align*}
where $c_1$ and $c_2$ are some uniform constants. Note that $\zeta_1$ holds because all the $\bu_t^*$, $\bD_{t+1}$, $\bM_t$ and $\bz_{t+1}$ are uniformly bounded (see Proposition~\ref{prop:bound:veABMDu}), and $\zeta_2$ holds because $\ell(\bz_{t+1}, \bD_{t+1})$ and $\tilde{g}_t(\bD_{t+1})$ are uniformly bounded (see Corollary~\ref{coro:bound l lip gt}).

For the third term on the right hand side of \eqref{eq:pf:diff b_t}, we can similarly show that
\begin{align*}
 \abs{ f_{t+1}(\bD_{t+1}) - f_t(\bD_{t+1}) } \leq&\ \abs{ \tilde{f}_{t+1}(\bD_{t+1}) - \tilde{f}_t(\bD_{t+1}) } + \frac{c_3}{t+1}\\
=&\ \abs{ \frac{1}{t+1} \ell(\bD_{t+1}; \bz_{t+1}) - \frac{1}{t+1} \tilde{f}_t(\bD_{t+1}) } + \frac{c_3}{t+1}\\
\stackrel{\zeta_3}{\leq}&\ \frac{c_4}{t+1},
\end{align*}
where $c_3$ and $c_4$ are some uniform constants, and $\zeta_3$ holds as $\ell(\bD_{t+1}; \bz_{t+1})$ and $\tilde{f}_t(\bD_{t+1})$ are both uniformly bounded (see Corollary~\ref{coro:bound lip ft}).

Using Corollary~\ref{coro:bound l lip gt} and Corollary~\ref{coro:bound lip ft}, we know that both $g_t(\bD)$ and $f_t(\bD)$ are uniformly Lipschitz. That is, there exist uniform constants $\kappa_1$, $\kappa_2$, such that
\begin{align*}
\abs{ g_t(\bD_{t+1}) - g_t(\bD_t) } \leq&\ \kappa_1 \fronorm{\bD_{t+1} - \bD_t} \stackrel{\zeta_4}{\leq} \frac{\kappa_3}{t+1},\\
\abs{ f_t(\bD_{t+1}) - f_t(\bD_t) } \leq&\ \kappa_2 \fronorm{\bD_{t+1} - \bD_t} \stackrel{\zeta_5}{\leq} \frac{\kappa_4}{t+1}.
\end{align*}
Here, $\zeta_4$ and $\zeta_5$ are derived by applying Prop.~\ref{prop:diff_D} and $\kappa_3$ and $\kappa_4$ are some uniform constants. Finally, we have a bound for \eqref{eq:pf:diff b_t}:
\begin{align*}
\abs{ b_{t+1} - b_t } \leq \frac{\kappa_0}{t+1},
\end{align*}
where $\kappa_0$ is some uniform constant.

By applying Lemma~\ref{lem:mairal}, we conclude that $\{b_t\}_{t\geq 1}$ converges to zero. That is,
\begin{align*}
\lim_{t \rightarrow +\infty} g_t(\bD_t) - f_t(\bD_t) = 0.
\end{align*}
Since we have proved in Theorem~\ref{thm:convergence g_t(D_t)} that $g_t(\bD_t)$ converges almost surely, we conclude that $f_t(\bD_t)$ converges almost surely to the same limit of $g_t(\bD_t)$.
\end{proof}

\subsection{Proof of Theorem~\ref{thm:stationary}}
We need a technical result to prove Theorem~\ref{thm:stationary}.
\begin{proposition}
Let $f(\bD)$ be the expected loss function which is defined in \eqref{eq:f(D)}. Then, $f(\bD)$ is continuously differentiable and $\nabla f(\bD) = \EXP_{\bz}[\nabla_{\bD} \ell(\bD; \bz)]$. Moreover, $\nabla f(D)$ is uniformly Lipschitz on $\mathcal{D}$.
\end{proposition}
\begin{proof}
We have shown in Proposition~\ref{prop:l:Lipschtiz} that $\ell(\bD; \bz)$ is continuously differentiable, $f(D)$ is also continuously differentiable and we have $\nabla f(D) = \EXP_{\bz}[\nabla_D \ell(\bD; \bz)]$.

Next, we prove the Lipschitz of $\nabla f(D)$. Let $\bv'(\bD'; \bz')$ and $\be'(\bD'; \bz')$ be the minimizer of $\tl(\bD', \bv, \be; \bz')$. Since $\tl(\bD, \bv, \be; \bz)$ has a unique minimum and is continuous in $\bD$, $\bv$, $\be$ and $\bz$, $\bv'(\bD'; \bz')$ and $\be'(\bD'; \bz')$ is continuous in $\bD$ and $\bz$.

Let $\varLambda = \{ j \mid \be'_j \neq 0 \}$. According the first order optimality condition, we know that
\begin{align*}
\frac{\partial \tl(\bD, \bv, \be; \bz)}{\partial \be} = 0,
\end{align*}
which implies
\begin{align*}
\lambda_1(\bz' - \bD'\bv' - \be') \in \lambda_2 \partial\onenorm{\be'}.
\end{align*}
Hence,
\begin{align*}
\abs{ (\bz' - \bD'\bv' - \be')_j } = \frac{\lambda_2}{\lambda_1},\ \forall j \in \varLambda.
\end{align*}
Since $\bz - D\bv - \be$ is continuous in $\bz$ and $\bD$, there exists an open neighborhood $\mathcal{V}$, such that for all $(\bz'', D'') \in \mathcal{V}$, if $j \notin \varLambda$, then $\abs{(\bz'' - D''\bv'' - \be'')_j} < \frac{\lambda_2}{\lambda_1}$ and $\be''_j = 0$. That is, the support set of $\be'$ will not change.

Let us denote $\bH =[\bD\ \bI_p]$, $\br = [\bv\trans\ \be\trans]\trans$ and define the function
\begin{align*}
\tl(\bH_{\varLambda}, \br_{\varLambda}; \bz) = \fractwo{\lambda_1} \twonorm{ \bz - \bH_{\varLambda}\br_{\varLambda}}^2 +\fractwo{1} \twonorm{[\bI_d\ 0]\br_{\varLambda}}^2  + \lambda_2 \onenorm{[0\ \bI_{\abs{\varLambda}}]\br_{\varLambda}}.
\end{align*}
Above, $\br_{\varLambda} = [\bv\trans\ \be_{\varLambda}\trans]\trans$, and accordingly for $\bH_{\varLambda}$.
Since $\tl(D_{\varLambda}, \br_{\varLambda}; \bz)$ is strongly convex with respect to $\br_{\varLambda}$, there exists a uniform constant $\kappa_1$, such that for all $\br''_{\varLambda}$,
\begin{align}
\label{eq:pf:diff tl' 1}
\tl( H'_{\varLambda}, \br''_{\varLambda}; \bz') - \tl(H'_{\varLambda}, \br'_{\varLambda}; \bz') \geq \kappa_1 \twonorm{\br''_{\varLambda} - \br'_{\varLambda}}^2 = \kappa_1 \( \twonorm{\bv'' - \bv'}^2 + \twonorm{\be''_{\varLambda} - \be'_{\varLambda}}^2 \).
\end{align}
On the other hand,
\begin{align}
\label{eq:pf:diff tl' 2}
&\ \tl( \bH'_{\varLambda}, \br''_{\varLambda}; \bz') - \tl(\bH'_{\varLambda}, \br'_{\varLambda}; \bz') \notag\\
=&\ \tl(\bH'_{\varLambda}, \br''_{\varLambda}; \bz') -\tl(\bH''_{\varLambda}, \br''_{\varLambda}; \bz'')  + \tl( \bH''_{\varLambda}, \br''_{\varLambda}; \bz'') - \tl(\bD'_{\varLambda}, \br'_{\varLambda}; \bz') \notag\\
\leq&\ \tl(\bH'_{\varLambda}, \br''_{\varLambda}; \bz') -\tl(\bH''_{\varLambda}, \br''_{\varLambda}; \bz'')  + \tl(\bH''_{\varLambda}, \br'_{\varLambda}; \bz'') - \tl(\bH'_{\varLambda}, \br'_{\varLambda}; \bz'),
\end{align}
where the last inequality holds because $\br''$ is the minimizer of $\tl(\bH'', \br; \bz'')$.

We shall prove that $\tl(\bH'_{\varLambda}, \br_{\varLambda}; \bz') - \tl(\bH''_{\varLambda}, \br_{\varLambda}; \bz'')$ is Lipschitz w.r.t. $\br$, which implies the Lipschitz of $\bv'(\bD; \bz)$ and $\be'(\bD; \bz)$. By algebra, we have
\begin{align*}
\nabla_{\br} \( \tl(\bH'_{\varLambda}, \br_{\varLambda}; \bz') - \tl(\bH''_{\varLambda}, \br_{\varLambda}; \bz'')\) =&\ \lambda_1 \Big[ H'^{\top}_{\varLambda}(H'_{\varLambda} - H''_{\varLambda}) + (H'_{\varLambda} - H''_{\varLambda})^{\top} H''_{\varLambda} \\
&\ + H'^{\top}_{\varLambda}(\bz'' - \bz')  + (H''_{\varLambda} - H'_{\varLambda})^{\top} \bz'' \Big].
\end{align*}
Note that $\fronorm{H'_{\varLambda}}$, $\fronorm{H''_{\varLambda}}$ and $\bz''$ are all uniformly bounded by Assumption~\ref{as:z} and Prop.~\ref{prop:bound:veABMDu}. Hence, there exists uniform constants $c_1$ and $c_2$, such that
\begin{align*}
 \twonorm{ \nabla_{\br} \( \tl(\bz', H'_{\varLambda}, \br_{\varLambda}) - \tl(\bz'', H''_{\varLambda}, \br_{\varLambda})\) } \leq c_1 \fronorm{H'_{\varLambda} - H''_{\varLambda}} + c_2 \twonorm{\bz' - \bz''},
\end{align*}
which implies that $\tl(\bH'_{\varLambda}, \br_{\varLambda}; \bz') - \tl(\bH''_{\varLambda}, \br_{\varLambda}; \bz'')$ is Lipschitz w.r.t $\br_{\varLambda}$ where the Lipschitz coefficient $c(H'_{\varLambda}, H''_{\varLambda}, \bz', \bz'') = c_1 \fronorm{H'_{\varLambda} - H''_{\varLambda}} + c_2 \twonorm{\bz' - \bz''}$. Combining this fact with \eqref{eq:pf:diff tl' 1} and \eqref{eq:pf:diff tl' 2}, we obtain
\begin{align*}
\kappa_1 \twonorm{\br''_{\varLambda} - \br'_{\varLambda}}^2 \leq c(H'_{\varLambda}, H''_{\varLambda}, \bz', \bz'') \twonorm{\br''_{\varLambda} - \br'_{\varLambda}}.
\end{align*}
Therefore, $\br(\bD; \bz)$ is Lipschitz and so are $\bv(\bD; \bz)$ and $\be(\bD; \bz)$. Note that according to Proposition~\ref{prop:l:Lipschtiz},
\begin{align*}
 \nabla f(\bD') - \nabla f(\bD'')=&\ \EXP_{\bz}\Big[(\bH' \br' - \bz) (\bv')\trans - (\bH''\br'' - \bz) (\bv'')\trans\Big]\\
=& \EXP_{\bz}\Big[ \bH'\br'(\bv' - \bv'')\trans + (\bH' - \bH'')\br' (\bv'')\trans \\
&\ + \bH''(\br' - \br'')(\bv'')\trans + \bz(\bv'' - \bv')\trans \Big].
\end{align*}
Thus,
\begin{align*}
&\ \fronorm{\nabla f(D') - \nabla f(D'')}\\
\stackrel{\zeta_1}{\leq}&\ \EXP_{\bz}\Big[ \twonorm{H'\br'} \twonorm{\bv' - \bv''} + \fronorm{H' - H''}  \fronorm{\br'\bv''^{\top}} \\
&\ + \fronorm{H''} \twonorm{\br' - \br''}  \twonorm{\bv''} + \twonorm{\bz}  \twonorm{\bv' - \bv''} \Big]\\
\stackrel{\zeta_2}{\leq}&\ \EXP_{\bz} \Big[ (\gamma_1 + \gamma_2 \twonorm{\bz}) \fronorm{H' - H''} \Big]\\
\stackrel{\zeta_3}{\leq}&\ \gamma_0 \fronorm{D' - D''},
\end{align*}
where $\gamma_0$, $\gamma_1$ and $\gamma_2$ are all uniform constants. Here, $\zeta_1$ holds due to the convexity of $\fronorm{\cdot}$. $\zeta_2$ is derived by using the result that $\br(\bH; \bz)$ and $\bv(\bH; \bz)$ are both Lipschitz and $H'$, $H''$, $\br'$, $\br''$, $\bv'$ and $\bv''$ are all uniformly bounded. $\zeta_3$ holds because $\bz$ is uniformly bounded and $\fronorm{\bH' - \bH''} = \fronorm{\bD' - \bD''}$. Thus, we complete the proof.
\end{proof}

\begin{proof}
(Proof of Theorem~\ref{thm:stationary}) 
Since $\frac{1}{t}\bA_t$ and $\frac{1}{t} \bB_t$ are uniformly bounded (Proposition~\ref{prop:bound:veABMDu}), there exist sub-sequences of $\{ \frac{1}{t}\bA_t \}$ and $ \{\frac{1}{t} \bB_t \}$ that converge to $A_{\infty}$ and $B_{\infty}$ respectively. Then $\bD_t$ will converge to $\bD_{\infty}$. Let $\bW$ be an arbitrary matrix in $\Rpd$ and $\{ h_k \}_{k\geq 1}$ be any positive sequence that converges to zero.

As $g_t(\bD)$ is a surrogate function of $f_t(\bD)$, for all $t$ and $k$, we have
\begin{align*}
g_t(\bD_t + h_k \bW) \geq f_t(\bD_t + h_k \bW).
\end{align*}
Let $t$ tend to infinity, and note that $f(D) = \lim_{t \rightarrow \infty} f_t(D)$, we have
\begin{align*}
g_{\infty}(D_{\infty} + h_k W) \geq f(D_{\infty} + h_k W).
\end{align*}

Note that the Lipschitz property of $\nabla f(\bD)$ indicates that the second derivative of $f(\bD)$ is uniformly bounded. By a simple calculation, we can also show that it also holds for $g_t(\bD)$. This fact implies that we can take the first order Taylor expansion for both $g_t(\bD)$ and $f(\bD)$ even when $t$ tends to infinity (because the second order derivatives of them always exist). That is,
\begin{align*}
\tr(h_k \bW\trans \nabla g_{\infty}(\bD_{\infty})) + o(h_k \fronorm{\bW}) \geq \tr(h_k \bW\trans \nabla f(\bD_{\infty})) + o(h_k \fronorm{\bW}).
\end{align*}
Above, we use the fact that $\lim_{t \rightarrow \infty} g_t(\bD_t) - f(\bD_t) = 0$ as implied by Corollary~\ref{cor:diff f}. By multiplying $\frac{1}{h_k \fronorm{\bW}}$ on both sides and note that $\{ h_k \}_{k\geq 1}$ is a positive sequence,  it follows that
\begin{align*}
\tr\(\frac{1}{\fronorm{\bW}} \bW\trans \nabla g_{\infty}(D_{\infty})\) + \frac{o(h_k \fronorm{\bW})}{h_k \fronorm{\bW}} \geq \tr\(\frac{1}{\fronorm{\bW}} \bW\trans \nabla f(D_{\infty})\) + \frac{o(h_k \fronorm{\bW})}{h_k \fronorm{\bW}}.
\end{align*}
Now let $k$ go to infinity, we obtain
\begin{align*}
\tr\(\frac{1}{\fronorm{\bW}} \bW\trans \nabla g_{\infty}(\bD_{\infty})\)  \geq \tr\(\frac{1}{\fronorm{\bW}} \bW\trans \nabla f(\bD_{\infty})\).
\end{align*}
Note that this inequality holds for any matrix $\bW \in \Rpd$, so we actually have
\begin{align*}
\nabla g_{\infty}(\bD_{\infty}) = \nabla f(\bD_{\infty}).
\end{align*}
As $\bD_{\infty}$ is the minimizer of $g_{\infty}(\bD)$, we have
\begin{align*}
\nabla f(\bD_{\infty}) = \nabla g_{\infty}(\bD_{\infty}) = 0.
\end{align*}
The proof is complete.
\end{proof}

\clearpage

\bibliographystyle{abbrvnat}
\bibliography{OLRR}

\end{document}